\documentclass{article}

     \usepackage[nonatbib,final]{neurips_2019}

\usepackage[utf8]{inputenc} %
\usepackage[T1]{fontenc}    %
\usepackage[pagebackref=true,breaklinks=true,letterpaper=true,colorlinks,bookmarks=false]{hyperref}
\usepackage{url}            %
\usepackage{booktabs}       %
\usepackage{amsfonts}       %
\usepackage{nicefrac}       %
\usepackage{microtype}      %
 \usepackage{multirow}
\usepackage{pifont}
\usepackage{amssymb}
\usepackage{amsmath}
\usepackage{xcolor}
\usepackage{graphicx}
\usepackage{subcaption}
\usepackage{xspace}
\usepackage{wrapfig}
\usepackage{algorithm}
\usepackage{algorithmic}
\usepackage{bm}

\usepackage{pgfplots}
\usepgfplotslibrary{statistics} %

\usepackage{pgfplots}
\usepackage{mathtools}
\usepackage{amsthm}
\usepackage{mathtools}
\usepackage{mathrsfs}
\usepackage{thmtools}
\usepackage{xfrac}
\declaretheorem{theorem}

\declaretheorem[sibling=theorem]{lemma}
\declaretheorem[sibling=theorem]{corollary}
\declaretheorem[sibling=theorem]{definition}

\usepackage{pgfplots}
\usepackage{mathtools}
\usepackage{amsthm}
\usepackage{mathtools}
\usepackage{mathrsfs}
\usepackage{thmtools}
\usepackage{xfrac}

\newcommand{\PP}{\mathbb{P}}

\newcommand{\RR}{\mathbb{R}}

\newcommand{\nc}{\newcommand}
\nc{\rnc}{\renewcommand}
\nc{\lbar}[1]{\overline{#1}}
\nc{\bra}[1]{\langle#1|}
\nc{\ket}[1]{|#1\rangle}
\nc{\ketbra}[2]{|#1\rangle\!\langle#2|}
\nc{\braket}[2]{\langle#1,#2\rangle}
\nc{\proj}[1]{| #1\rangle\!\langle #1 |}
\nc{\avg}[1]{\langle#1\rangle}
\nc{\smfrac}[2]{\mbox{$\frac{#1}{#2}$}}
\nc{\tr}{\operatorname{tr}}
\nc{\pd}{\mathrm{P}}
\nc{\gtd}{\mathrm{D}}
\nc{\gfid}{\bar{\mathrm{F}}}

\nc{\tchern}{\left( \smfrac{7}{4} p - 1 \right)}
\nc{\tailprobchern}{0.995}
\nc{\splitpoint}{\smfrac{1}{4}}
\nc{\epsupper}{\smfrac{1}{4}}
\nc{\mconst}{16}

\newcommand{\I}{\ensuremath{\mathbb{I}}}

\newcommand{\prob}{\ensuremath{\mathbb{P}}}

\newcommand{\yhi}{\ensuremath{\hat{y}_i}}
\newcommand{\vzbar}{\overline{\ensuremath{z}}}

\newcommand{\what}{\ensuremath{\hat{w}}}

\newcommand{\thetastar}{\ensuremath{\theta^\star}}

\DeclarePairedDelimiter{\parens}{\lparen}{\rparen}
\DeclarePairedDelimiter{\brackets}{[}{]}
\DeclarePairedDelimiter{\ip}{\langle}{\rangle}

\newtheorem{thm}{Theorem}

\def\neurips{neurips}
\def\version{arxiv}  %

\DeclareMathOperator*{\argmax}{arg\,max}

\newcommand{\xmark}{\ding{55}}
\newcommand{\fgsm}{FGSM$^{20}$}
\newcommand{\multiTarget}{MultiTar.}
\newcommand{\multitargeted}{\texttt{MultiTargeted}}
\newcommand{\bigunsup}[1]{80m@#1K}
\newcommand{\bigunsupn}{80m@N}

\newcommand{\wresnet}[1]{WRN-#1}

\newcommand{\accnat}{\ensuremath{\mathcal{A}_{nat}}}
\newcommand{\accfgsm}{$\mathcal{A}_{{FGSM}^{20}}$}
\newcommand{\accmulti}{$\mathcal{A}_{{\multiTarget}}$}

\newcommand{\eg}{\textit{e.g.}\@\xspace}
\newcommand{\ie}{\textit{i.e.}\@\xspace}
\newcommand{\supdata}{\ensuremath{\mathcal{S}_{n}}}
\newcommand{\unsupdata}{\ensuremath{\mathcal{U}_{m}}}
\newcommand{\unsupdist}{\ensuremath{P(X)}}

\newcommand{\supdist}{\ensuremath{P(X,Y)}}
\newcommand{\labelset}{\ensuremath{\mathcal{Y}}}
\newcommand{\imset}{\ensuremath{\mathcal{X}}}

\newcommand{\norm}[1]{\left\lVert#1\right\rVert}

\newcommand{\mathev}{\mathop{\mathbb{E}}}

\newcommand{\normball}{\ensuremath{N_\epsilon}}
\newcommand{\loss}{\ensuremath{\mathcal{L}}}
\newcommand{\eps}{\ensuremath{\varepsilon}}

\newcommand{\bmx}{\ensuremath{\bm{x}}}
\newcommand{\bmy}{\ensuremath{\bm{y}}}
\newcommand{\empuatloss}{\ensuremath{\hat{\mathcal{L}}^{OT}}}
\newcommand{\empadvloss}{\ensuremath{\hat{\mathcal{L}}^{adv}}}

\newcommand{\smalldots}{\hbox to 1em{.\hss.\hss.}}

\newcommand{\tinyds}{80m}
\newcommand{\tinydsfullname}{80 Million Tiny Images}

\newcommand{\uatpp}{UAT++}
\newcommand{\uatft}{UAT-FT}
\newcommand{\uatot}{UAT-OT}
\newcommand{\uatftest}{\ensuremath{\hat{w}}}
\newcommand{\cifar}{CIFAR-10}
\newcommand{\svhn}{SVHN}
\newcommand{\lossot}{\mathcal{L}^{OT}_{\text{unsup}}}
\newcommand{\losssup}{\mathcal{L}_{\text{sup}}}
\newcommand{\xent}{\texttt{xent}}
\hypersetup{
	plainpages=false,
	colorlinks=true,              %
	linkcolor=blue,               %
	anchorcolor=blue,             %
	citecolor=blue,               %
	filecolor=blue,               %
	urlcolor=blue,                %
	pdfview=FitH,                 %
	pdfstartview=FitH,            %
	pdfpagelayout=SinglePage      %
}

\title{Are Labels Required for Improving \\ Adversarial Robustness?}

\author{
	\renewcommand*{\thefootnote}{\fnsymbol{footnote}}
	Jonathan Uesato\footnotemark[1]
\quad\quad\quad
	Jean-Baptiste Alayrac\footnotemark[1]
\quad\quad\quad
	Po-Sen Huang\footnotemark[1]
\\ \\
\textbf{Robert Stanforth} 
\quad\quad\quad
\textbf{Alhussein Fawzi}
\quad\quad\quad
\textbf{Pushmeet Kohli}
	\\ \\
	DeepMind
	\\
	\texttt{\{juesato,jalayrac,posenhuang\}@google.com}
}

\begin{document}

\maketitle

\renewcommand{\thefootnote}{\fnsymbol{footnote}}
\footnotetext[1]{Equal contribution, random order.}
\ifx\version\neurips
\footnotetext[2]{The authors declare that the present paper is independent of ``Unlabeled Data Improves Adversarial Robustness''~\cite{Carmon_robustness}.}
\fi
\renewcommand*{\thefootnote}{\arabic{footnote}}

\begin{abstract}
Recent work has uncovered the interesting (and somewhat surprising) finding that training models to be invariant to adversarial perturbations requires substantially larger datasets than those required for standard classification. 
This result is a key hurdle in the deployment of robust machine learning models in many real world applications where labeled data is expensive. 
Our main insight is that \emph{unlabeled} data can be a competitive alternative to labeled data for training adversarially robust models.
Theoretically, we show that in a simple statistical setting, the sample complexity for learning an adversarially robust model from unlabeled data matches the fully supervised case up to constant factors.
On standard datasets like \cifar, a simple Unsupervised Adversarial Training (UAT) approach using unlabeled data improves robust accuracy by $21.7\%$ over using 4K supervised examples alone, and captures over $95\%$ of the improvement from the same number of labeled examples.
Finally, we report an improvement of 4$\%$ over the previous state-of-the-art on \cifar{} against the strongest known attack by using additional unlabeled data from the uncurated \tinydsfullname{} dataset.
This demonstrates that our finding extends as well to the more realistic case where unlabeled data is also uncurated, therefore opening a new avenue for improving adversarial training. 
\end{abstract}

\section{Introduction}
Deep learning has revolutionized many areas of research such as natural language processing, speech recognition or computer vision.
System based on these techniques are now being developed and deployed for a wide variety of applications, from recommending/ranking content on the web~\cite{covington2016deep, huang2013learning} to autonomous driving~\cite{e2eselfdrivingcar} and even in medical diagnostics~\cite{de2018clinically}. The safety-critical nature of some of these tasks necessitates the need for ensuring that the deployed models are robust and generalize well to all sorts of variations that can occur in the inputs.	
Yet, it has been shown that the commonly used deep learning models are vulnerable to adversarial perturbations in the input~\cite{Szegedy14Intrig}, \eg it is possible to fool an image classifier into predicting arbitrary classes by carefully choosing perturbations imperceptible to the human eye.

Since the discovery of these results, many approaches have been developed to prevent this type of behaviour. 
One of the most effective and popular approaches is known as \emph{supervised adversarial training}~\cite{goodfellow2014explaining,madry18iclr} which works by generating adversarial samples in an online manner through an inner optimization procedure and then using them to augment the standard training set.
Despite substantial work in this space, accuracy of classifiers on adversarial inputs remains much lower than that on normal inputs. 
Recent theoretical work has offered a reason for this discrepancy, and argues that training models to be invariant to adversarial perturbations requires substantially larger datasets than those required for the standard classification task~\cite{Schmidt18moredata}.

This result is a key hurdle in the development and deployment of robust machine learning models in many real world applications where 
labeled data is expensive.
Our central hypothesis is that additional unlabeled examples may suffice for adversarial robustness.
Intuitively, this is based on two related observations, explained in Sections \ref{sec:uat} and \ref{sec:label_noise_analysis}.
First, adversarial robustness depends on the smoothness of the classifier around natural images, which can be estimated from unlabeled data.
Second, only a relatively small amount of labeled data is needed for standard generalization.
Thus, if adversarial training is robust to label noise, labels from supervised examples can be propagated to unsupervised examples to train a smoothed classifier with improved adversarial robustness.

Motivated by this, we explore Unsupervised Adversarial Training (UAT) to use unlabeled data for adversarial training.
We study this algorithm in a simple theoretical setting, proposed by~\cite{Schmidt18moredata} to study adversarial generalization. We show that once we are given a single labeled example, the sample complexity of UAT matches the fully supervised case up to constant factors.
In independent and concurrent work, \cite{Carmon_robustness, Zhai_robustness, Najafi_robustness} also study the use of unlabeled data for improving adversarial robustness, which we discuss in Section \ref{sec:related_work}.

Experimentally, we find strong support for our main hypothesis.
On \cifar{} and SVHN, with very limited annotated data, our method reaches robust accuracies of 54.1\% and 84.4\% respectively against the \fgsm~attack~\cite{kurakin17advScale}.
These numbers represent a significant improvement over purely supervised approaches (32.5\% and 66.0\%) on same amount of data and almost match methods that have access to full supervision (55.5\% and 86.2\%), capturing over 95\% of their improvement, without labels.
Further, we show that we can successfully leverage realistically obtained unsupervised and uncurated data to improve the state-of-the-art on \cifar{} at $\eps=8/255$ from 52.58\% to 56.30\% against the strongest known attack.

	\noindent
	\textbf{Contributions.}
\textbf{(i)} In Section~\ref{sec:uat}, we propose a simple and theoretically grounded strategy, UAT, to leverage unsupervised data for adversarial training. \textbf{(ii)} We provide, in Section~\ref{subsec:exp_control}, strong empirical support for our initial hypothesis that \emph{unlabeled data can be competitive with labeled data when it comes to training adversarially robust classifiers}, therefore opening a new avenue for improving adversarial training. 
\textbf{(iii)} Finally, by leveraging noisy and uncurated data obtained from web queries, we set a new state-of-the-art on \cifar{} without depending on any additional human labeling.
\vspace*{-0.3cm}	
\section{Related Work}
\label{sec:related_work}
\vspace*{-0.2cm}	
\noindent
\textbf{Adversarial Robustness.}
\cite{biggio2013evasion,Szegedy14Intrig} observed that neural networks which achieve extremely high accuracy on a randomly sampled test set may nonetheless be vulnerable to adversarial examples, or small but highly optimized perturbations of the data which cause misclassification.
Since then, many papers have proposed a wide variety of defenses to the so-called adversarial attacks, though few have proven robust against stronger attacks \cite{athalye2018obfuscated, carlini2017adversarial, uesato2018adversarial}.
One of the most successful approaches for obtaining classifiers that are adversarially robust is adversarial training~\cite{athalye2018obfuscated, uesato2018adversarial}.
Adversarial training directly minimizes the adversarial risk by approximately solving an inner maximization problem by projected gradient descent (PGD) to generate small perturbations that increase the prediction error, and uses these perturbed examples for training~\cite{goodfellow2014explaining, kurakin17advScale, madry18iclr}.
The TRADES approach in \cite{Zhang19Trades} improves these results by instead minimizing a surrogate loss which upper bounds the adversarial risk.
Their objective is very similar to the one used in \uatot{} (UAT with online targets introduced in Section~\ref{subsec:uat}) but is estimated purely on labeled rather than unlabeled data.

Common to all these approaches, a central challenge is adversarial generalization. 
For example, on \cifar{} with a perturbation of $\eps=8/255$, the adversarially trained model in \cite{madry18iclr} achieves an adversarial accuracy of $46\%$, despite near $100\%$ adversarial accuracy on the train set.
For comparison, standard models can achieve natural accuracy of $96\%$ on \cifar{} \cite{zagoruyko2016wide}.
Several recent papers have studied generalization bounds for adversarial robustness~\cite{attias2018improved,khim2018adversarial,Schmidt18moredata,yin2018rademacher}.
Of particular relevance to our work, \cite{Schmidt18moredata} argues that adversarial generalization may require more data than natural generalization.
One solution explored in \cite{hendrycks19pre} is to use pretraining on ImageNet, a large supervised dataset, to improve adversarial robustness.
In this work, we study whether more labeled data is necessary, or whether unlabeled data can suffice.
While to our knowledge, this question has not been directly studied, several works such as \cite{gu2014towards, samangouei2018defense, song2017pixeldefend} propose using generative models to detect or denoise adversarial examples, which can in principle be learned on unlabeled data.
However, so far, such approaches have not proven to be robust to strong attacks \cite{athalye2018obfuscated, uesato2018adversarial}.

\noindent
\textbf{Semi-supervised learning.}
Learning from unlabeled data is an active area of research.
The semi-supervised learning approach~\cite{chapelleSSL} which, in addition to labeled data, also uses unlabeled data to learn better models is particularly relevant to our work.
One of the most effective technique for semi-supervised learning is to use smoothness regularization: the model is trained to be invariant to small perturbation applied to unsupervised samples~\cite{bachman14pseudoensemble,berthelot19mixmatch,laine17tessl,Miyato17VAT,Sajjadi16,Xie19UAD}.
Of particular relevance to UAT, \cite{Miyato17VAT} also uses adversarial perturbations to smooth the model outputs.
In addition, co-training \cite{blum1998combining} and recent extensions~\cite{chen18trinet,rebuffi19ssl} use the most confident predictions on unlabeled data to iteratively construct additional labeled training data.
These work all focus on improving standard generalization whereas we explore the use of similar ideas in the context of adversarial generalization.

\noindent
\textbf{Semi-supervised learning for adversarial robustness.}
The observation that adversarial robustness can be optimized without labels was made independently and concurrently by \cite{Carmon_robustness, Najafi_robustness, Zhai_robustness}. 
Of particular interest, \cite{Carmon_robustness} proposes a meta-algorithm Robust Self-Training (RST), similar to UAT. Indeed, the particular instantiation of RST used in \cite{Carmon_robustness} and the fixed-target variant of UAT are nearly equivalent: the difference is whether the base algorithm minimizes the robust loss from \cite{Zhang19Trades} or the vanilla adversarial training objective \cite{madry18iclr}. Their results also provide strong, independent evidence that unlabeled and uncurated examples improve robustness on both \cifar{} and \svhn.

\section{Unsupervised Adversarial Training (UAT)}
	\label{sec:uat}
	
	In this section, we introduce and motivate our approach, Unsupervised Adversarial Training (UAT), which enables the use of unlabeled data to train robust classifiers.

	\noindent
	\textbf{Notation.}
	Consider the classification problem of learning a predictor $f_\theta$ to map inputs $x\in\imset$ to labels $y\in\labelset$.
	In this work, $f$ is of the form: $f_\theta(x)=\argmax_{y\in\labelset} p_\theta(y | x)$, where $p_\theta(.|x)$ is parameterized by a neural network.
	We assume data points $(x, y)$ are i.i.d. samples from the data-generating joint distribution $\supdist$ over $\imset \times \labelset$.
	$\unsupdist$ denotes the unlabeled distribution over $\imset$ obtained by marginalizing out $Y$ in $\supdist$.
	We assume access to a labeled training set $\supdata = \{(x_i, y_i)\}_{1 \leq i \leq n}$, where $(x_i, y_i) \sim \supdist$ and an unlabeled training set $\unsupdata = \{x_i\}_{1 \leq i \leq m}$, where $x_i \sim \unsupdist$.
	
	\textbf{Evaluation of Adversarial Robustness.}
	The \textit{natural risk} is $\loss_{nat}(\theta) = \mathbb{E}_{(x, y) \sim \supdist} \: \ell(y, f_\theta(x))$, where $\ell$ is the $0{-}1$ loss.
	Our primary objective is minimizing \textit{adversarial risk}: 
	$\loss_{adv}(\theta) = \mathev_{ \supdist } \, \sup_{x' \in \normball(x)} \ell(y, f_{\theta}(x'))$. 
	As is common, the neighborhood $\normball(x)$ is taken in this work to be the $L_\infty$ ball: $\normball(x) = \{x' : \norm{x' - x}_\infty \leq \epsilon \}$. 
	Because the inner maximization cannot be solved exactly, we report the surrogate adversarial risk $\loss_g(\theta) = \mathev_{ \supdist} \: \ell(f_\theta(x'), y)$, 
	where $x' = g(x, y, \theta)$ is an approximate solution to the inner maximization computed by some fixed adversary $g$.
	Typically, $g$ is (a variant of) projected gradient descent (PGD) with a fixed number of iterations.
	
	\subsection{Unsupervised Adversarial Training (UAT)}
	\label{subsec:uat}

\textbf{Motivation.}
As discussed in the introduction, a central challenge for adversarial training has been the difficulty of adversarial generalization.
Previous work has argued that adversarial generalization may simply require more data than natural generalization.
We ask a simple question: is more \emph{labeled} data necessary, or is \emph{unsupervised} data sufficient?
This is of particular interest in the common setting where unlabeled examples are dramatically cheaper to acquire than labeled examples ($m \gg n$).
For example, for large-scale image classification problems, unlabeled examples can be acquired by scraping images off the web, whereas gathering labeled examples requires hiring human labelers.

We now consider two algorithms to study this question. Both approaches are simple -- we emphasize the point that large unlabeled datasets can help bridge the gap between natural and adversarial generalization.
Later, in Sections \ref{sec:gaussian_model} and \ref{sec:experiments}, we show that both in a simple theoretical model and empirically, unlabeled data is in fact \emph{competitive} with labeled data. In other words, for a fixed number of additional examples, we observe similar improvements in adversarial robustness regardless of whether or not they are labeled.

	\noindent
	\textbf{Strategy 1: Unsupervised Adversarial Training  with Online Targets (UAT-OT).} 
We note that adversarial risk can be bounded as 
$\loss_{adv} = \loss_{nat} + (\loss_{adv} - \loss_{nat}) \leq \loss_{nat} + \mathev_{ \supdist } \, \sup_{x' \in \normball(x)} \ell(f_\theta(x'), f_\theta(x))$,
similarly to the decomposition in \cite{Zhang19Trades}.
We refer to the first term as the \emph{classification} loss, and the second terms as the \emph{smoothness} loss.
Even for adversarially trained models,
it has been observed that the smoothness loss dominates the classification loss on the test set, 
suggesting that controlling the smoothness loss is the key to adversarial generalization.
For example, the adversarially trained model in \cite{madry18iclr} achieves natural accuracy of 87\% but adversarial accuracy of 46\% on \cifar{} at $\eps=8/255$.

Notably, the smoothness loss has no dependence on labels, and thus can be minimized purely through unsupervised data.
	UAT-OT directly minimizes a differentiable surrogate of the smoothness loss on the unlabeled data.
	Formally, we use the loss introduced in~\cite{Miyato17VAT} and also used in \cite{Zhang19Trades}
	\begin{equation}
	\label{eq:uat_ot}
	\lossot(\theta) = \mathev_{x\sim\unsupdist} \sup_{x'\in\mathcal{N}_\epsilon(x)}\mathcal{D}(p_{\hat{\theta}}(.|x),p_\theta(.|x')),
	\end{equation}
	where $\mathcal{D}$ is the Kullback-Leibler divergence, 
	and $\hat{\theta}$ indicates a fixed copy of the parameters $\theta$ in order to stop the gradients from propagating.
	While \cite{Miyato17VAT}, which primarily focuses on natural generalization, uses a single step approximation of the inner maximization, we use an iterative PGD adversary, since prior work indicates strong adversaries are crucial for effective adversarial training \cite{kurakin17advScale,madry18iclr}.

	\noindent
\textbf{Strategy 2: Unsupervised Adversarial Training  with Fixed Targets (UAT-FT).} 
This strategy directly leverages the gap between standard generalization and adversarial generalization.
The main idea is to first train a \textit{base classifier} for standard generalization on the supervised set $\supdata$. 
Then, this model is used to estimate labels, hence \emph{fixed targets}, on the unsupervised set $\unsupdata$. 
This allows us to employ standard supervised adversarial training using these fixed targets.
Formally, it corresponds to using the following loss:
	\begin{equation}
\label{eq:unsup_FT}
\mathcal{L}^{FT}_{\text{unsup}}(\theta) = \mathev_{x\sim\unsupdist} \sup_{x'\in\mathcal{N}_\epsilon(x)}\xent(\hat{y}(x),p_\theta(.|x')),
\end{equation}
where \xent{} is the cross entropy loss and $\hat{y}(x)$ is a \emph{pseudo}-label obtained from a model trained for standard generalization on $\supdata$ alone.
Thus, provided a sufficiently large unlabeled dataset, \uatft{} recovers a smoothed version of the base classifier, which matches the predictions of the base classifier on clean data, while maintaining stability of the predictions within local neighborhoods of the data.
	
\noindent
\textbf{Overall training.}
For the overall objective, we use a weighted combination of the supervised loss and the chosen unsupervised loss, controlled by a hyperparameter $\lambda$:
$\mathcal{L}(\theta)=\losssup(\theta)+\lambda\mathcal{L}_{\text{unsup}}(\theta)$.
The unsupervised loss can be either  $\lossot{}$ (UAT-OT), $\mathcal{L}^{\text{FT}}_{\text{unsup}}$ (UAT-FT) or both (\uatpp).
Finally, note that the unsupervised loss can also be used on the samples of the supervised set by simply adding the $x_i$'s of $\supdata$ in $\unsupdata$.
The pseudocode and implmenetation details are described in Appendix \ref{sec:implementation_notes}.  

\subsection{Theoretical model}
\label{sec:gaussian_model}

To improve our understanding of the effects of unlabeled data, we study the simple setting proposed by \cite{Schmidt18moredata} to analyze the required sample complexity of adversarial robustness.

\begin{definition}[Gaussian model \cite{Schmidt18moredata}]
    Let $\theta^* \in \mathbb{R}^d$ be the per-class mean vector and let $\sigma > 0$ be the variance parameter.
    Then the $(\theta^*, \sigma)$-Gaussian model is defined by the following distribution
    over $(x, y) \in \mathbb{R}^d \times \{ \pm 1 \}$: 
    First, draw a label $y \in \{ \pm 1 \}$ uniformly at random. 
    Then sample the data point $x \in \mathbb{R}^d$ from $\mathcal{N}(y \cdot \theta^*, \sigma^2 I)$.
\end{definition}

In \cite{Schmidt18moredata}, this setting was chosen to model the empirical observation that adversarial generalization requires more data than natural generalization.
They provide an algorithm which achieves fixed, arbitrary (say, $1\%$) accuracy using a single sample.
However, to achieve the same adversarial accuracy, they show that any algorithm requires at least $c_1 \eps^2 \sqrt{d} \, / \log d$ samples and provide an algorithm requiring $n \geq c_2 \eps^2 \sqrt{d}$ samples, for fixed constants $c_1, c_2$.

Here, we show that this sample complexity can be dramatically improved by replacing labeled examples with unlabeled samples.
We first define an analogue of \uatft{} to leverage unlabeled data in this setting. 
For training an adversarially robust classifier, the algorithm in \cite{Schmidt18moredata} computes a sample mean of per-point estimates.
We straightforwardly adapt this procedure for unlabeled data, as in \uatft: we first estimate a base classifier from the labeled examples, then compute a sample mean using fixed targets from this base classifier.

\begin{definition}[Gaussian UAT-FT]
\label{def:uat_ft_gaussian}
	Given $n$ labeled examples $(x_1, y_1), \dots, (x_n, y_n)$ and $m$ unlabeled examples $x_{n+1}, \dots, x_{n+m}$,
    let $\hat{w}_\text{sup}$ denote the sample mean estimator on labeled examples:
    $\hat{w}_\text{sup} = \sum_{i=1}^n y_i x_i$.
    The \uatft{} estimator is then defined as the sample mean $\uatftest = \sum_{i=n+1}^{n+m} \hat{y}_i x_i$
    where $\hat{y}_i = f_{\hat{w}_\text{sup}}(x_i)$.
\end{definition}

Theorem \ref{thm:uat_ft_gaussian} states that in contrast to the purely supervised setting which requires $O(\sqrt{d} \, / \log d \, )$ examples,
in the semi-supervised setting, a single labeled example, along with $O(\sqrt{d} \, )$ examples are sufficient to achieve fixed, arbitrary accuracy.

\begin{thm}
\label{thm:uat_ft_gaussian}
	Consider the $(\theta^*, \sigma)$-Gaussian model with $\norm{\theta^*}_2 = \sqrt{d}$ and $\sigma \leq \frac{1}{32}d^{\sfrac{1}{4}}$.
    Let $\uatftest$ be the the \uatft{} estimator as in Definition \ref{def:uat_ft_gaussian}. 
    Then with high probability, for $n=1$, the linear classifier $f_{\uatftest}$ has $\ell_{\infty}^{\epsilon}$-robust classification error at most $1\%$ if
    $$m \geq c \epsilon^2 \sqrt{d}$$
\end{thm}

where $c$ is a fixed, universal constant. The proof is deferred to Appendix \ref{sec:proof_gaussian_model}.
For ease of comparison, we consider the same Gaussian model parameters $\norm{\theta^*}_2$ and $\sigma$ as used in \cite{Schmidt18moredata}.
The sample complexity in Theorem \ref{thm:uat_ft_gaussian} matches the sample complexity of the algorithm provided in \cite{Schmidt18moredata} up to constant factors, despite using unlabeled rather than labeled examples.
We now turn to empirical investigation of whether this result is reflected in practical settings.

\section{Experiments}
\label{sec:experiments}

In section~\ref{subsec:exp_control}, we first investigate our primary question: for adversarial robustness, can unlabeled examples be competitive with labeled examples?
These operate in the standard semi-supervised setting where we use a small fraction of the original training set as $\supdata$, and provide varying amounts of the remainder as $\unsupdata$.
After observing high robustness, particularly for \uatft{} and \uatpp, we run several controlled experiments in section~\ref{sec:label_noise_analysis} to understand why this approach works well.
In section~\ref{sec:real_world_dist_shift}, we explore the robustness of UAT to shift in the distribution $\unsupdist$.
Finally, we use UAT to improve existing state-of-the-art adversarial robustness on \cifar, using the \tinydsfullname{} dataset as our source of unlabeled data.

\subsection{Adversarial robustness with few labels}
\label{subsec:exp_control}

\textbf{Experimental setup.} 
We run experiments on the \cifar{} and \svhn{} datasets, with $L_\infty$ constraints of $\eps=8/255$ and $\eps=0.01$ respectively, which are standard for studying adversarial robustness of image classifiers \cite{gowal2018effectiveness,madry18iclr,Zhang19Trades,wong18scaling}.
For adversarial evaluation, we report against 20-step iterative FGSM \cite{kurakin17advScale}, for consistency with previous state-of-the-art \cite{Zhang19Trades}.
In our later experiments for Section \ref{subsec:exp_80m}, we also evaluate against a much stronger attack, \multitargeted~\cite{gowal2019alternative}, which provides a more accurate proxy for the adversarial risk.
As we demonstrate in Appendix \ref{sec:multitargeted_attack}, the \multitargeted{} attack is significantly stronger than an expensive PGD attack with random restarts, which is in turn significantly stronger than \fgsm.
We follow previous work \cite{madry18iclr,Zhang19Trades} for our choices of model architecture, data preprocessing, and hyperparameters, which are detailed in Appendix \ref{sec:experimental_details}.

To study the effect of unlabeled data, we randomly split the existing training set into a small supervised set $\supdata$ and use the remaining $N-n$ training examples as a source of unlabeled data.
We also split out $10000$ examples from the training set to use as validation, for both \cifar{} and \svhn, since neither dataset comes with a validation set.
We then study the effect on robust accuracy of increasing $m$, the number of unsupervised samples, across different regimes ($m \approx n$ vs. $m \gg n$).

\noindent
\textbf{Baselines.}
 We compare results with the two strongest existing supervised approaches, standard adversarial training~\cite{madry18iclr} and TRADES~\cite{Zhang19Trades}, which do not use unsupervised data.
We also compare to VAT~\cite{Miyato17VAT}, which was designed for \emph{standard} semi-supervised learning but can be adapted for unsupervised adversarial training as explained in Appendix~\ref{app:vat}. 
Finally, to compare the benefits of labeled and unlabeled data, we compare to the \emph{supervised oracle}, which represents the best possible performance, where the model is provided the ground-truth label even for samples from \unsupdata.

\noindent
\subsubsection{Main results}
\label{sec:control_exp_analysis}

We first test the hypothesis that for adversarial robustness, additional unlabeled data is competitive with additional labeled data. 
Figure \ref{fig:control_exp} summarizes the results.
We report the adversarial accuracy for varying $m$, when $n$ is fixed to $4000$ and $1000$, for \cifar{} and \svhn{} respectively.

\textbf{Comparison to baselines.}
All models show significant improvements in adversarial robustness over the baselines for all numbers of unsupervised samples.
With the maximum number (32k / 60k) of unlabeled images, even the weakest UAT model, \uatot{}, shows {12.9\% / 16.9\%} %
improvement over the baselines not leveraging unlabeled data, and {6.4\% / 1.6\%} %
improvement over VAT on \cifar{} and SVHN, respectively.

\textbf{Comparison between UAT variants.} 
We compare the results of 3 different UAT variants: \uatot{}, \uatft{}, and \uatpp{}. 
Comparing \uatft{} and \uatot{}, when there are larger number of unsupervised samples, we observe that the \uatft{} shows a significant improvement compared to \uatot{} on \cifar{}, \uatot{} performs similarly to \uatft{} on SVHN.
With smaller numbers of unsupervised samples, the two approaches perform similarly.
Empirically, we observe that \uatpp{}, which combines the two approaches, outperforms either individually.
We thus primarily use \uatpp{} for our later experiments.

\textbf{Comparison to the oracle.}
Figure \ref{fig:control_exp} provides strong support for our main hypothesis. In particular, we observe that when using large unsupervised datasets, \uatpp{} performs nearly as well as the supervised oracle.
In Fig. \ref{fig:control_exp}a, with 32K unlabeled examples, \uatpp{} achieves {54.1\%} on \cifar{}, which is {1.4\%} %
lower than the supervised oracle. 
Similarly, with 60K unlabeled data, in Fig. \ref{fig:control_exp}b,  \uatpp{} achieves {84.4\%} on SVHN which is {1.8\%} lower than the supervised oracle. %

\textbf{Conclusion.} We demonstrate that, leveraging large amounts of unlabeled examples, \uatpp{} achieves similar adversarial robustness to supervised oracle, which uses label information.
In particular, without requiring labels, \uatpp{} captures over {97.6\% / 97.9\%} %
of the improvement from 32K / 60K additional examples compared with supervised oracle on \cifar{} and SVHN, respectively.

\begin{figure*}[t]
	\centering
	\begin{subfigure}[t]{.485\linewidth}
		\includegraphics[height=5.0cm]{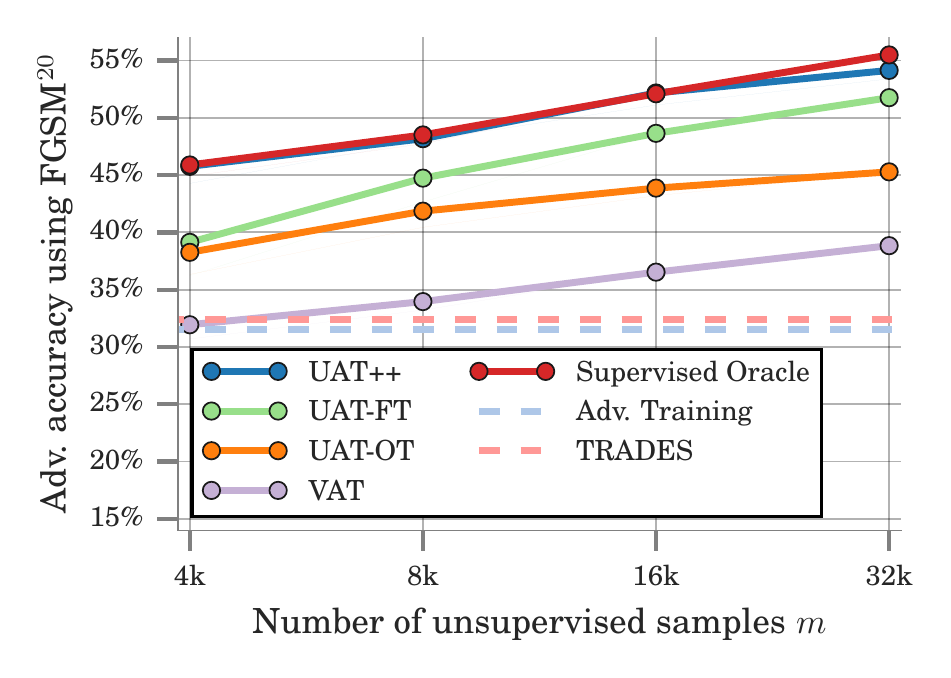} %
	\end{subfigure}
	\hfill
	\begin{subfigure}[t]{.485\linewidth}
		\includegraphics[height=5.0cm]{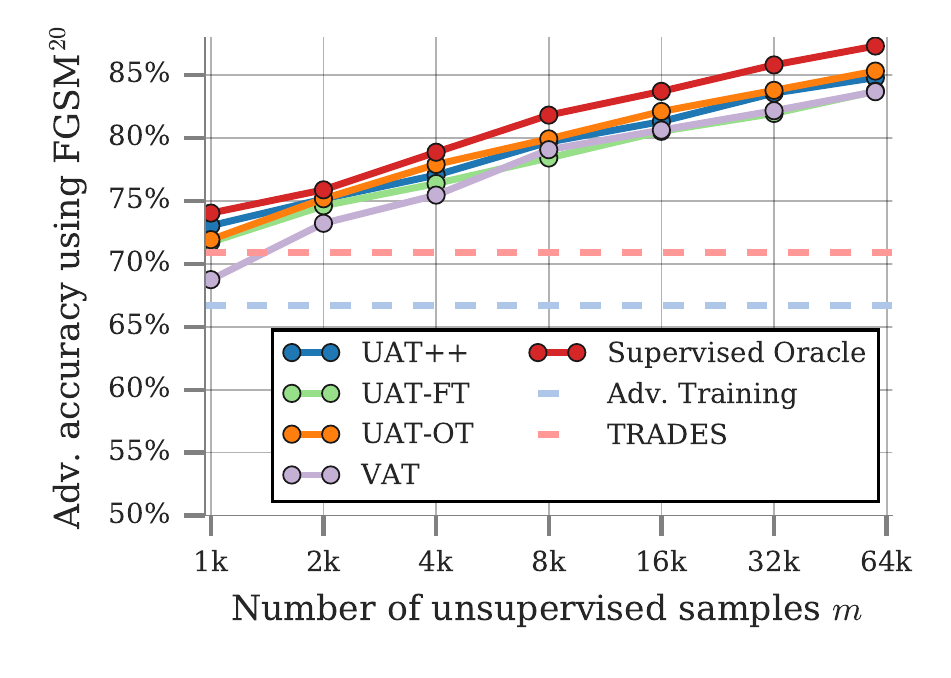} %
	\end{subfigure}
	\caption{\small Comparison of labeled data and unsupervised data for improving adversarial generalization on \cifar{} (\textbf{left,a}) and SVHN (\textbf{right,b})\label{fig:control_exp}} 
\end{figure*}

\subsubsection{Label noise analysis}
\label{sec:label_noise_analysis}

Given the effectiveness of \uatft{} and \uatpp, we perform an ablation study on the impact of label noise on UAT for adversarial robustness.

\textbf{Experimental setup.} 
To do so, we first divide the \cifar{} training set into halves, where the first 20K examples are used for training the base classifier and the latter 20K are used to train a UAT model.
Of the latter 20K, we treat 4K examples as labeled, as in Section \ref{sec:control_exp_analysis}, and the remaining 16K as unlabeled.
We consider two different approaches to introducing label noise.
For \uatft{} (Correlated), we produce pseudo-labels using the \uatft{} procedure, where the number of training examples used for the base classifier varies between between 500 and 20K.
This produces base classifiers with error rates between 7\% and 48\%.
For \uatft{} (Random), we randomly flip the label to a randomly selected incorrect class.
The results are shown in Figure \ref{fig:label_noise}.

\textbf{Analysis.}
In Fig. \ref{fig:label_noise}a, in the \uatft{} (Random) case, adversarial accuracy is relatively flat between 1\% and 20\%.
Even with 50\% of the examples mislabeled, the decrease in robust accuracy is less than 10\%. %
At the highest level of noise,  \uatft{} still obtains a 8.0\% %
improvement in robustness accuracy over the strongest baseline which does not exploit unsupervised data.
Similarly, in the \uatft{} (Correlated) case, robust accuracy is relatively flat between 7\% and 23\% noise level, and even at 48\% corrupted labels, \uatft{} outperforms the purely supervised baselines by 6.3\%. %

To understand these results, we believe that the main function of the unsupervised data in UAT is to improve generalization of the smoothness loss, rather than the classification loss.
While examples with corrupted labels have limited utility for improving classification accuracy, they can still be leveraged to improve the smoothness loss.
This is most obvious in \uatot{}, which has no dependence on the predicted labels (and is thus a flat line in Figure \ref{fig:label_noise}a).
However, Figure \ref{fig:label_noise}a supports the hypothesis that \uatft{} also works similarly, given its effectiveness even in cases where up to half of the labels are corrupted.
As mentioned in Section \ref{sec:uat}, because generalization gap of the classification loss is typically already small, controlling generalization of the smoothness loss is key to improved adversarial robustness.

\textbf{Comparison to standard generalization.}
We compare the robustness of \uatft{} to label noise, to an analogous pseudo-labeling technique applied to natural generalization.
Comparing between Figures \ref{fig:label_noise}a and \ref{fig:label_noise}b, we observe that with increasing label noise, the rate of degradation in robustness of adversarial trained models is much lower than the rate of degradation in accuracy of models obtained with standard training.
In particular, while standard training procedures can be robust to random label noise, as observed in previous work \cite{patrini17robusNoise,rolnick17LabelNoise},
accuracy decreases almost one-to-one (slope -0.78) with correlated errors.
This is natural, as with a very large unsupervised dataset, we expect to recover the base classifier (modulo the 4k additional supervised examples).

\textbf{Conclusion.} 
UAT shows significant robustness to label noise, achieving an 8.0\% improvement over the baseline even with nearly 50\% error in the base classifier.
We hypothesize that this is primarily because UAT operates primarily on the smoothness loss, rather than the classification loss, and is thus less dependent on the pseudo-labels.

\begin{figure*}[t!]
	\centering
	\begin{subfigure}[t]{.48\linewidth}
		\centering
		\includegraphics[height=4.9cm]{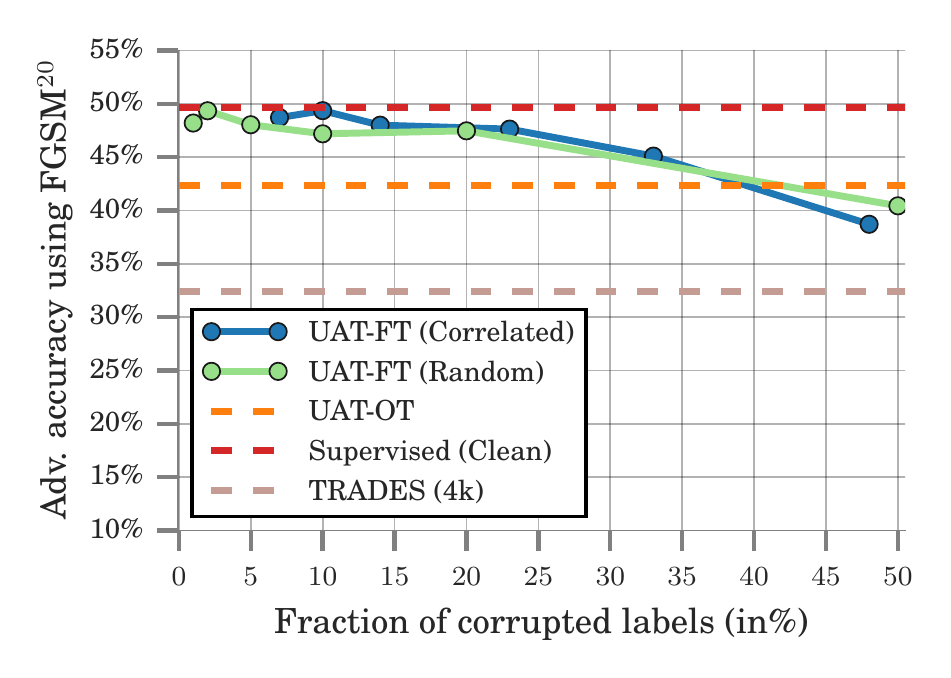} %
	\end{subfigure}
	\hfill
	\begin{subfigure}[t]{.48\linewidth}
		\centering
		\includegraphics[height=4.9cm]{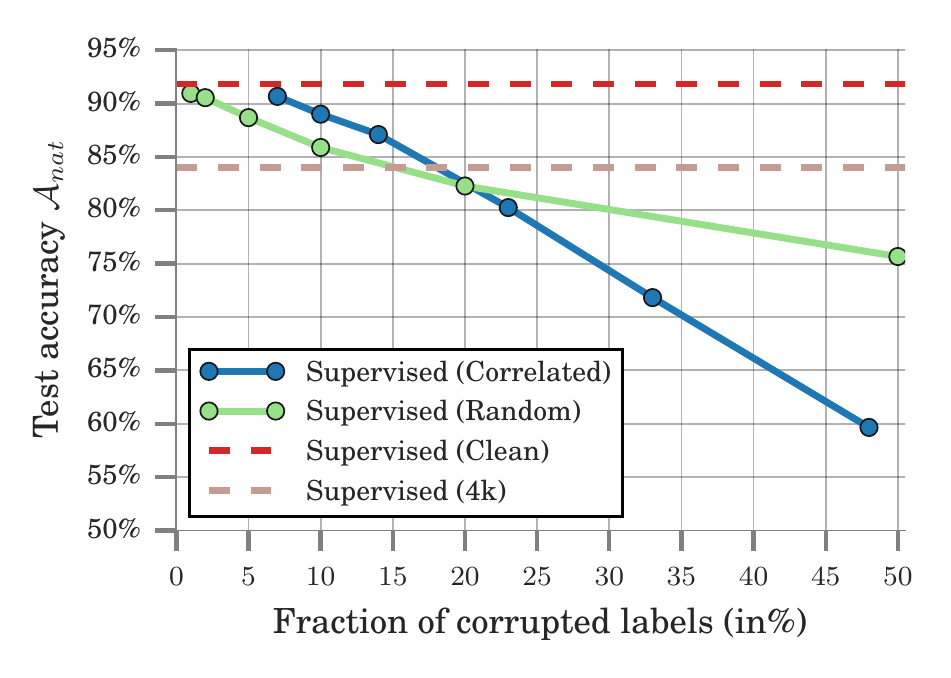} %
	\end{subfigure}
	\caption{\small Effects of label noise on adversarial \textbf{(left, a}) and natural \textbf{(right, b)} accuracies, on \cifar{} \label{fig:label_noise}} 
\end{figure*}

\subsection{Unsupervised data with distribution shift}
\label{sec:real_world_dist_shift}

\noindent
\textbf{Motivation.}
In Section~\ref{subsec:exp_control}, we studied the standard semi-supervised setting, where \unsupdist{} is the marginal of the joint distribution \supdist{}.
As pointed out in~\cite{Oliver18Realistic}, real-world unlabeled datasets may involve varying degrees of distribution shift from the labeled distribution.
For example, images from \cifar{}, even without labels, required human curation to not only restrict to images of the choosen 10 classes but also to ensure that selected images were photo-realistic (line drawings were rejected) or that only one instance of the object was present (see Appendix C of~\cite{cifar10} for the full labeler instruction sheet).
We thus study whether our approach is robust to such distribution shift, allowing us to fully leverage data which is not only unlabeled, but also uncurated.

\noindent
We use the \textbf{\tinydsfullname{}}~\cite{80m} dataset (hereafter, \tinyds{}) as our uncurated data source, a large dataset obtained by web queries for 75,062 words.
Because collecting this dataset required no human filtering, it provides a perfect example of uncurated data that is cheaply available at scale.
Notably, \cifar{} is a human-labeled subset of \tinyds, which has been restricted to 10 classes.

\noindent
\textbf{Preprocessing.}
Because the majority of \tinyds{} contains images distinct from the \cifar{} classes, we apply an automated filtering technique similar to~\cite{Xie19UAD}, detailed in Appendix \ref{app:dataset_details}.
Briefly, we first restrict to images obtained from web queries matching the \cifar{} classes, and filter out near duplicates of the \cifar{} test set using GIST features \cite{oliva2001modeling,douze2009evaluation}.
For each class, we rank the images based on the prediction confidence from a WideResNet-28-10 model pretrained on the \cifar{} dataset. 
We then take the top 10k, 20k, or 50k images per class, to create the \bigunsup{100}, \bigunsup{200}, and \bigunsup{500} datasets, respectively.

\noindent
\textbf{Overview.}
We first conduct a preliminary study on the impact of distribution shift in a low data regime in Section~\ref{sec:dist_shift}, and we finally demonstrate how UAT can be used to leverage large scale realistic uncurated data in Section~\ref{subsec:exp_80m}.

\subsubsection{Preliminary study: Low data regime}
\label{sec:dist_shift}

To study the effect of having unsupervised data from a different distribution, we repeat the same experimental setup described in Section~\ref{sec:control_exp_analysis} where we draw \unsupdata{} from \bigunsup{200} rather than \cifar{}.
Results are given in Figure \ref{fig:dist_shift}.
\begin{wrapfigure}{R}{0.44\textwidth}
  \vspace*{-0.25cm}
	\begin{center}
  		\includegraphics[width=0.44\textwidth]{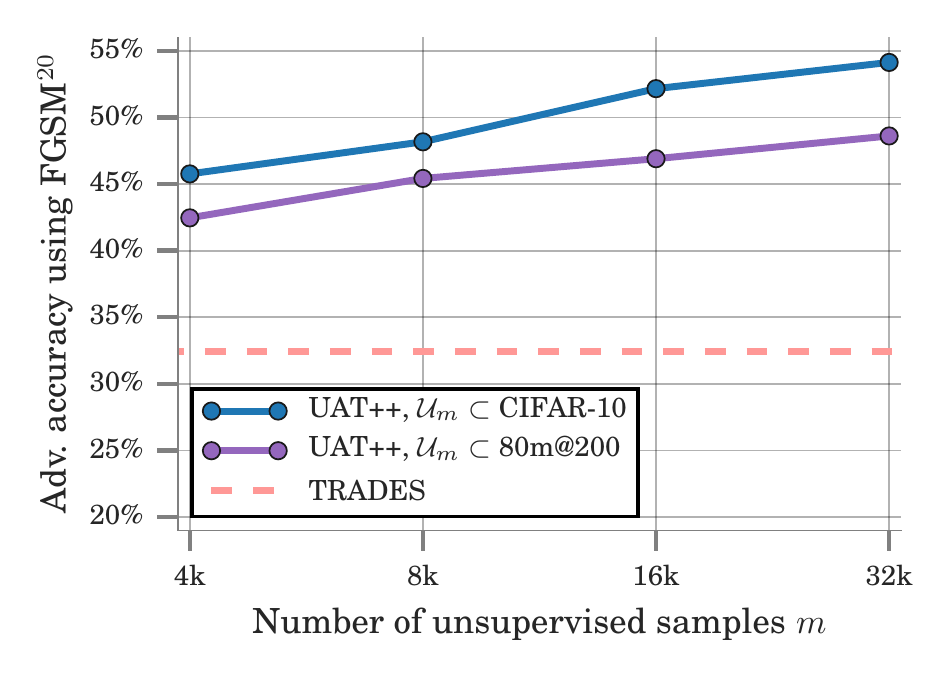}
  \caption{Distribution shift on \cifar \label{fig:dist_shift}}
	\end{center}
  \vspace{-1cm}
\end{wrapfigure}
For simplicity, we report our best performing method, UAT++, in both settings: $\unsupdata \, {\subset}$ \bigunsup{200} and $\unsupdata\,{\subset}$ \cifar{}.
First, we observe that when using 32K images from unsupervised data, either  \uatpp{} (\bigunsup{200}) or  \uatpp{} (\cifar{}) outperforms the baseline, TRADES \cite{Zhang19Trades}, which only uses the 4K supervised examples.
Specifically, \uatpp{} with \bigunsup{200}  achieves 48.6\% robust accuracy, a 16.2\% %
improvement over TRADES.
On the other hand, \uatpp{} performs substantially better when \unsupdata{} is drawn from \cifar{} rather than \bigunsup{200}, by a margin of 5.5\% %
with 32K unlabeled examples.

\textbf{Conclusion.} While unlabeled data from the same distribution is significantly better than off-distribution unlabeled data, the off-distribution unlabeled data is still much better than no unsupervised data at all.
In the next section, we explore scaling up the off-distribution case.

\subsubsection{Large scale regime}
\label{subsec:exp_80m}

\begin{table}[]
	\centering
	\resizebox{.9\textwidth}{!}{%
\begin{tabular}{@{}rcccccc@{}}
\toprule
Method                              & Sup. Data & Unsup. Data    & Network       & \accnat & \accfgsm & \accmulti \\ \midrule
\cite{wong18scaling}                & \cifar{}   & \xmark         & -         & 27.07\% & 23.54\% & - \\
AT \cite{madry18iclr}                  & \cifar{}   & \xmark         & \wresnet{28}         & 87.30\% & 47.04\% & 44.54\%\\
\cite{zheng16stabTraining}          & \cifar{}   & \xmark         & -         & 94.64\% & 0.15\%  & -\\
\cite{kurakin17advScale}            & \cifar{}   & \xmark         & -         & 85.25\% & 45.89\% & -\\
\cite{hendrycks19pre} & ImageNet + \cifar{}   & \xmark   & \wresnet{28}  & 87.1\% & 57.40\% & $\leq$52.9\%*\\
AT-Reimpl.~\cite{madry18iclr} & \cifar{}   & \xmark         & \wresnet{34}  & 87.08\% & 52.93\% & 47.10\%\\
TRADES~\cite{Zhang19Trades}         & \cifar{}   & \xmark         & \wresnet{34}  & 84.92\% & 57.11\% & 52.58\%\\ %
\midrule
UAT++                               & \cifar{}   & \bigunsup{100} & \wresnet{34}  & 86.04\% & 59.41\% & 52.64\% \\
UAT++                               & \cifar{}  & \bigunsup{200} & \wresnet{34}  & 85.85\% & \textbf{62.18\%} & \textbf{53.35\%}\\
UAT++                               & \cifar{}   & \bigunsup{500} & \wresnet{34}  & 78.34\% & 58.04\% & 48.99\% \\\midrule

UAT++                               & \cifar{}   & \bigunsup{200} & \wresnet{70}  &  86.75\%   & 62.89\% & 55.04\%  \\
UAT++                               & \cifar{}   & \bigunsup{200} & \wresnet{106} & 86.46\% & \textbf{63.65\%}   & \textbf{56.30\%} \\ %

\bottomrule
\end{tabular}
	}
\vspace*{0.1cm}
\caption{\small Experimental results using 80m Tiny Images dataset (as a unsupervised data)  and \cifar{} (as supervised data), where \accnat~represents the original test accuracy, \accfgsm~represents the adversarial accuracy under 20 step FGSM, and \accmulti~represents the adversarial accuracy under the strong \texttt{MultiTargeted} attack. \wresnet{$k$} denotes the Wide-ResNet with depth $k$. `*' indicates it is from~\cite{hendrycks19pre} using 100 PGD steps with 1000 random restarts, an attack that we have found to be weaker than the \texttt{MultiTargeted} attack. \label{tab:uat_largeScale}}
\end{table}

We now study whether uncurated data alone can be leveraged to improve the state-of-the-art for adversarial robustness on \cifar.
For these experiments, we use subsets of \tinyds{} in conjunction with the full \cifar{} training set.
Table \ref{tab:uat_largeScale} summarizes the results. 
We report adversarial accuracies against two attacks.
First, we consider the FGSM \cite{goodfellow2014explaining, kurakin17advScale} attack with 20 steps (\fgsm) to allow for direct comparison with previous state-of-the-art \cite{Zhang19Trades}. 
Second, we evaluate against the \texttt{MultiTargeted} attack, which we find to be significantly stronger than the commonly used PGD attack with random restarts. Details are provided in Appendix~\ref{sec:multitargeted_attack}.

\noindent
\textbf{Baselines.}
For baseline models, we evaluate the models released by \cite{madry18iclr,Zhang19Trades}.
For fair comparison with our setup, we also reimplement adversarial training (AT-Reimpl. \cite{madry18iclr}) using the same attack we use for UAT, which we found to be slightly more efficient than the original attack.
This is detailed in Appendix~\ref{app:pgd_details}.
We also compare to~\cite{hendrycks19pre}, which uses more \emph{labeled} data by pretraining on ImageNet.
All other reported numbers are taken from \cite{Zhang19Trades}. 

\noindent
\textbf{Comparison with same model.}
First, we compare \uatpp{} with three different sets of unsupervised data (\bigunsup{100}, \bigunsup{200}, and \bigunsup{500}) using the same model architecture (\wresnet{34}) as in TRADES.
In all cases, we outperform TRADES under \fgsm. 
When using \bigunsup{200}, we improve upon TRADES by 5.07\% under \fgsm and 0.77\% under the \texttt{MultiTargeted} attack.  
We note the importance of leveraging more unsupervised data when going from \bigunsup{100} to \bigunsup{200}.
However, performance degrades when using \bigunsup{500} which we attribute to the fact that \bigunsup{500} contains significantly more out-of-distribution images.
Finally, comparing with the recent work of~\cite{hendrycks19pre}, we note that using more unsupervised data can outperform using additional supervised data for pretraining.

\noindent
\textbf{Further analysis.}
We run several additional checks against gradient masking \cite{athalye2018obfuscated, tramer2017ensemble, uesato2018adversarial}, detailed in Appendix~\ref{app:adv_eval_details}.
We show that a gradient-free attack, SPSA \cite{uesato2018adversarial}, does not lower accuracy compared to untargeted PGD (Appendix~\ref{app:spsa_eval}),
visualize loss landscapes (Appendix~\ref{sec:loss_landscape}),
and empirically analyze attack convergence (Appendix~\ref{sec:attack_convergence}).
Overall, we do not find evidence that other attacks could outperform the \multitargeted{} attack.

\noindent
\textbf{A new state-of-the-art on \cifar{}.}
Finally, when using these significantly larger training sets, we observe significant underfitting, where robust accuracy is low even on the training set.
We thus also explore using deeper models.
We observe that \uatpp{} trained on the \bigunsup{200} unsupervised dataset using \wresnet{106} achieves state-of-the-art performance, +6.54\% under \fgsm and +3.72\% against \texttt{MultiTargeted} attack, compared to TRADES \cite{Zhang19Trades}.
Our trained model is available on our repository.\footnote{\url{https://github.com/deepmind/deepmind-research/tree/master/unsupervised_adversarial_training}}

\section{Conclusion}	
Despite the promise of adversarial training, its reliance on large numbers of labeled examples has presented a major challenge towards developing robust classifiers.
In this paper, we hypothesize that annotated data might not be as important as commonly believed for training adversarially robust classifiers. 
To validate this hypothesis, we introduce two simple UAT approaches which we tested on two standard image classification benchmarks.  
These experiments reveal that indeed, one can reach near state-of-the-art adversarial robustness with as few as 4K labels for \cifar{} (10 times less than the original dataset) and as few as 1K labels for SVHN (100 times less than the original dataset). 
Further, we demonstrate that our method can also be applied to uncurated data obtained from simple web queries.
This approach improves the state-of-the-art on \cifar{} by 4\% against the strongest known attack. 
These findings open a new avenue for improving adversarial robustness using unlabeled data.
We believe this could be especially important for domains such as medical applications, where robustness is essential and gathering labels is particularly costly~\cite{finlayson19medicaladv}.

\bigskip
\textbf{Acknowledgements.}
We would like to especially thank Sven Gowal for helping us evaluate with the \texttt{MultiTargeted} attack and for the loss landscape visualizations, as well as insightful discussions throughout this project. 
We would also like to thank Andrew Zisserman, Catherine Olsson, Chongli Qin, Relja Arandjelovi\'{c}, Sam Smith, Taylan Cemgil, Tom Brown, and Vlad Firoiu for helpful discussions throughout this work.

{\small
	\bibliographystyle{ieee}
	\bibliography{biblio}

\begin{thebibliography}{10}\itemsep=-1pt

\bibitem{athalye2018obfuscated}
A.~Athalye, N.~Carlini, and D.~Wagner.
\newblock Obfuscated gradients give a false sense of security: Circumventing
  defenses to adversarial examples.
\newblock {\em arXiv preprint arXiv:1802.00420}, 2018.

\bibitem{attias2018improved}
I.~Attias, A.~Kontorovich, and Y.~Mansour.
\newblock Improved generalization bounds for robust learning.
\newblock In {\em ALT}, 2018.

\bibitem{bachman14pseudoensemble}
P.~Bachman, O.~Alsharif, and D.~Precup.
\newblock Learning with pseudo-ensembles.
\newblock In {\em NeurIPS}, 2014.

\bibitem{berthelot19mixmatch}
D.~Berthelot, N.~Carlini, I.~Goodfellow, N.~Papernot, A.~Oliver, and C.~Raffel.
\newblock {MixMatch: A Holistic Approach to Semi-Supervised Learning}.
\newblock {\em arXiv:1905.02249}, 2019.

\bibitem{biggio2013evasion}
B.~Biggio, I.~Corona, D.~Maiorca, B.~Nelson, N.~{\v{S}}rndi{\'c}, P.~Laskov,
  G.~Giacinto, and F.~Roli.
\newblock Evasion attacks against machine learning at test time.
\newblock In {\em Joint European Conference on Machine Learning and Knowledge
  Discovery in Databases}, pages 387--402. Springer, 2013.

\bibitem{blum1998combining}
A.~Blum and T.~Mitchell.
\newblock Combining labeled and unlabeled data with co-training.
\newblock In {\em Proceedings of the Eleventh Annual Conference on
  Computational Learning Theory}, pages 92--100. ACM, 1998.

\bibitem{e2eselfdrivingcar}
M.~Bojarski, D.~D. Testa, D.~Dworakowski, B.~Firner, B.~Flepp, P.~Goyal, L.~D.
  Jackel, M.~Monfort, U.~Muller, J.~Zhang, X.~Zhang, J.~Zhao, and K.~Zieba.
\newblock End to end learning for self-driving cars.
\newblock {\em CoRR}, abs/1604.07316, 2016.

\bibitem{carlini2017adversarial}
N.~Carlini and D.~Wagner.
\newblock Adversarial examples are not easily detected: Bypassing ten detection
  methods.
\newblock In {\em Proceedings of the 10th ACM Workshop on Artificial
  Intelligence and Security}, pages 3--14. ACM, 2017.

\bibitem{carlini2017towards}
N.~Carlini and D.~Wagner.
\newblock Towards evaluating the robustness of neural networks.
\newblock In {\em Security and Privacy (SP), 2017 IEEE Symposium on}, pages
  39--57. IEEE, 2017.

\bibitem{Carmon_robustness}
Y.~{Carmon}, A.~{Raghunathan}, L.~{Schmidt}, P.~{Liang}, and J.~C. {Duchi}.
\newblock {Unlabeled Data Improves Adversarial Robustness}.
\newblock In {\em NeurIPS}, 2019.

\bibitem{chapelleSSL}
O.~Chapelle, B.~Scholkopf, and A.~Zien.
\newblock {\em Semi-Supervised Learning}.
\newblock MITPress, 2009.

\bibitem{chen18trinet}
D.-D. Chen, W.~Wang, W.~Gao, and Z.-H. Zhou.
\newblock Tri-net for semi-supervised deep learning.
\newblock In {\em IJCAI}, 2018.

\bibitem{covington2016deep}
P.~Covington, J.~Adams, and E.~Sargin.
\newblock Deep neural networks for {Y}outube recommendations.
\newblock In {\em Proceedings of the 10th ACM Conference on Recommender
  Systems}, pages 191--198. ACM, 2016.

\bibitem{de2018clinically}
J.~De~Fauw, J.~R. Ledsam, B.~Romera-Paredes, S.~Nikolov, N.~Tomasev,
  S.~Blackwell, H.~Askham, X.~Glorot, B.~O’Donoghue, D.~Visentin, et~al.
\newblock Clinically applicable deep learning for diagnosis and referral in
  retinal disease.
\newblock {\em Nature Medicine}, 24(9):1342, 2018.

\bibitem{douze2009evaluation}
M.~Douze, H.~J{\'e}gou, H.~Sandhawalia, L.~Amsaleg, and C.~Schmid.
\newblock Evaluation of gist descriptors for web-scale image search.
\newblock In {\em Proceedings of the ACM International Conference on Image and
  Video Retrieval}, page~19. ACM, 2009.

\bibitem{finlayson19medicaladv}
S.~G. Finlayson, J.~D. Bowers, J.~Ito, J.~L. Zittrain, A.~L. Beam, and I.~S.
  Kohane.
\newblock {Adversarial attacks on medical machine learning}.
\newblock {\em Science}, 2019.

\bibitem{goodfellow2014explaining}
I.~J. Goodfellow, J.~Shlens, and C.~Szegedy.
\newblock Explaining and harnessing adversarial examples.
\newblock {\em arXiv preprint arXiv:1412.6572}, 2014.

\bibitem{gowal2018effectiveness}
S.~Gowal, K.~Dvijotham, R.~Stanforth, R.~Bunel, C.~Qin, J.~Uesato,
  R.~Arandjelovic, T.~Mann, and P.~Kohli.
\newblock On the effectiveness of interval bound propagation for training
  verifiably robust models.
\newblock {\em arXiv preprint arXiv:1810.12715}, 2018.

\bibitem{gowal2019alternative}
S.~Gowal, J.~Uesato, C.~Qin, P.-S. Huang, T.~Mann, and P.~Kohli.
\newblock An alternative surrogate loss for pgd-based adversarial testing.
\newblock 2019.

\bibitem{gu2014towards}
S.~Gu and L.~Rigazio.
\newblock Towards deep neural network architectures robust to adversarial
  examples.
\newblock {\em arXiv preprint arXiv:1412.5068}, 2014.

\bibitem{hendrycks19pre}
D.~Hendrycks, K.~Lee, and M.~Mazeika.
\newblock Using pre-training can improve model robustness and uncertainty.
\newblock In {\em ICML}, 2019.

\bibitem{huang2013learning}
P.-S. Huang, X.~He, J.~Gao, L.~Deng, A.~Acero, and L.~Heck.
\newblock Learning deep structured semantic models for web search using
  clickthrough data.
\newblock In {\em Proceedings of the 22nd ACM international conference on
  Information \& Knowledge Management}, pages 2333--2338. ACM, 2013.

\bibitem{khim2018adversarial}
J.~Khim and P.-L. Loh.
\newblock Adversarial risk bounds for binary classification via function
  transformation.
\newblock {\em arXiv preprint arXiv:1810.09519}, 2018.

\bibitem{kingma2014adam}
D.~P. Kingma and J.~Ba.
\newblock Adam: A method for stochastic optimization.
\newblock {\em arXiv preprint arXiv:1412.6980}, 2014.

\bibitem{cifar10}
A.~Krizhevsky and G.~Hinton.
\newblock Learning multiple layers of features from tiny images.
\newblock Technical report, 2009.

\bibitem{kurakin17advScale}
A.~Kurakin, I.~Goodfellow, and S.~Bengio.
\newblock {Adversarial Machine Learning at Scale}.
\newblock In {\em ICLR}, 2017.

\bibitem{laine17tessl}
S.~Laine and T.~Aila.
\newblock Temporal ensembling for semi-supervised learnings.
\newblock In {\em ICLR}, 2017.

\bibitem{laurent2000adaptive}
B.~Laurent and P.~Massart.
\newblock Adaptive estimation of a quadratic functional by model selection.
\newblock {\em Annals of Statistics}, pages 1302--1338, 2000.

\bibitem{liu2016delving}
Y.~Liu, X.~Chen, C.~Liu, and D.~Song.
\newblock Delving into transferable adversarial examples and black-box attacks.
\newblock {\em arXiv preprint arXiv:1611.02770}, 2016.

\bibitem{madry18iclr}
A.~Madry, A.~Makelov, L.~Schmidt, D.~Tsipras, and A.~Vladu.
\newblock {Towards Deep Learning Models Resistant to Adversarial Attacks}.
\newblock In {\em ICLR}, 2018.

\bibitem{miller1995wordnet}
G.~A. Miller.
\newblock Wordnet: a lexical database for english.
\newblock {\em Communications of the ACM}, 38(11):39--41, 1995.

\bibitem{Miyato17VAT}
T.~Miyato, S.~ichi Maeda, M.~Koyama, and S.~Ishii.
\newblock {Virtual Adversarial Training: A Regularization Method for Supervised
  and Semi-Supervised Learning}.
\newblock {\em TPAMI}, 2018.

\bibitem{Najafi_robustness}
A.~{Najafi}, S.-i. {Maeda}, M.~{Koyama}, and T.~{Miyato}.
\newblock {Robustness to Adversarial Perturbations in Learning from Incomplete
  Data}.
\newblock {\em arXiv preprint arXiv:1905.13021}, May 2019.

\bibitem{oliva2001modeling}
A.~Oliva and A.~Torralba.
\newblock Modeling the shape of the scene: A holistic representation of the
  spatial envelope.
\newblock {\em IJCV}, 42(3):145--175, 2001.

\bibitem{Oliver18Realistic}
A.~Oliver, A.~Odena, C.~Raffel, E.~D. Cubuk, and I.~J. Goodfellow.
\newblock {Realistic Evaluation of Deep Semi-Supervised Learning Algorithms}.
\newblock In {\em NeurIPS}, 2018.

\bibitem{patrini17robusNoise}
G.~Patrini, A.~Rozza, A.~Menon, R.~Nock, and L.~Qu.
\newblock Making deep neural networks robust to label noise: a loss correction
  approach.
\newblock In {\em CVPR}, 2017.

\bibitem{rebuffi19ssl}
S.-A. Rebuffi, S.~Ehrhardt, K.~Han, A.~Vedaldi, and A.~Zisserman.
\newblock Semi supervised learning with scarce annotations.
\newblock In {\em ICML}, 2019.

\bibitem{rolnick17LabelNoise}
D.~Rolnick, A.~Veit, S.~J. Belongie, and N.~Shavit.
\newblock Deep learning is robust to massive label noise.
\newblock {\em CoRR}, abs/1705.10694, 2017.

\bibitem{Sajjadi16}
M.~Sajjadi, M.~Javanmardi, and T.~Tasdizen.
\newblock {Regularization with stochastic transformations and perturbations for
  deep semi-supervised learning}.
\newblock In {\em NeurIPS}, 2016.

\bibitem{samangouei2018defense}
P.~Samangouei, M.~Kabkab, and R.~Chellappa.
\newblock Defense-gan: Protecting classifiers against adversarial attacks using
  generative models.
\newblock {\em arXiv preprint arXiv:1805.06605}, 2018.

\bibitem{Schmidt18moredata}
L.~Schmidt, S.~Santurkar, D.~Tsipras, K.~Talwar, and A.~Madry.
\newblock {Adversarially Robust Generalization Requires More Data}.
\newblock In {\em NeurIPS}, 2018.

\bibitem{song2017pixeldefend}
Y.~Song, T.~Kim, S.~Nowozin, S.~Ermon, and N.~Kushman.
\newblock Pixeldefend: Leveraging generative models to understand and defend
  against adversarial examples.
\newblock {\em arXiv preprint arXiv:1710.10766}, 2017.

\bibitem{Szegedy14Intrig}
C.~Szegedy, W.~Zaremba, I.~Sutskever, J.~Bruna, D.~Erhan, I.~Goodfellow, and
  R.~Fergus.
\newblock Intriguing properties of neural networks.
\newblock In {\em ICLR}, 2014.

\bibitem{80m}
A.~Torralba, R.~Fergus, and W.~T. Freeman.
\newblock 80 million tiny images: a large dataset for non-parametric object and
  scene recognition.
\newblock {\em TPAMI}, 2008.

\bibitem{tramer2017ensemble}
F.~Tram{\`e}r, A.~Kurakin, N.~Papernot, I.~Goodfellow, D.~Boneh, and
  P.~McDaniel.
\newblock Ensemble adversarial training: Attacks and defenses.
\newblock {\em arXiv preprint arXiv:1705.07204}, 2017.

\bibitem{uesato2018adversarial}
J.~Uesato, B.~O'Donoghue, A.~v.~d. Oord, and P.~Kohli.
\newblock Adversarial risk and the dangers of evaluating against weak attacks.
\newblock In {\em ICML}, 2018.

\bibitem{ulyanov2016instance}
D.~Ulyanov, A.~Vedaldi, and V.~Lempitsky.
\newblock Instance normalization: The missing ingredient for fast stylization.
\newblock {\em arXiv preprint arXiv:1607.08022}, 2016.

\bibitem{wong18scaling}
E.~Wong, F.~R. Schmidt, J.~H. Metzen, and J.~Z. Kolter.
\newblock Scaling provable adversarial defenses.
\newblock In {\em NeurIPS}, 2018.

\bibitem{wu2018group}
Y.~Wu and K.~He.
\newblock Group normalization.
\newblock In {\em Proceedings of the European Conference on Computer Vision
  (ECCV)}, pages 3--19, 2018.

\bibitem{Xie19UAD}
Q.~Xie, Z.~Dai, E.~Hovy, M.-T. Luong, and Q.~V. Le.
\newblock {Unsupervised Data Augmentation}.
\newblock {\em arXiv:1904.12848}, 2019.

\bibitem{yin2018rademacher}
D.~Yin, K.~Ramchandran, and P.~Bartlett.
\newblock Rademacher complexity for adversarially robust generalization.
\newblock In {\em ICML}, 2018.

\bibitem{zagoruyko2016wide}
S.~Zagoruyko and N.~Komodakis.
\newblock Wide residual networks.
\newblock {\em arXiv preprint arXiv:1605.07146}, 2016.

\bibitem{Zhai_robustness}
R.~{Zhai}, T.~{Cai}, D.~{He}, C.~{Dan}, K.~{He}, J.~{Hopcroft}, and L.~{Wang}.
\newblock {Adversarially Robust Generalization Just Requires More Unlabeled
  Data}.
\newblock {\em arXiv preprint arXiv:1906.00555}, Jun 2019.

\bibitem{Zhang19Trades}
H.~Zhang, Y.~Yu, J.~Jiao, E.~P. Xing, L.~E. Ghaoui, and M.~I. Jordan.
\newblock {Theoretically Principled Trade-off between Robustness and Accuracy}.
\newblock {\em arXiv:1901.08573}, 2019.

\bibitem{zheng16stabTraining}
S.~Zheng, Y.~Song, T.~Leung, and I.~Goodfellow.
\newblock {Improving the Robustness of Deep Neural Networks via Stability
  Training}.
\newblock In {\em CVPR}, 2016.

\end{thebibliography}
}

\clearpage
\appendix
\section*{Overview}

The appendices are organised as follows. 
Appendix~\ref{sec:experimental_details} provides additional details on the experiments in Section~\ref{sec:experiments}.
Appendix~\ref{app:vat} details our implementation of VAT~\cite{Miyato17VAT} adapted for $L_\infty$ adversarial robustness.
Appendix~\ref{app:dataset_details} provides additional details on the 80m@N dataset generation procedure.
Appendix~\ref{app:code_release} details code release.
Appendix~\ref{app:adv_eval_details} includes additional experiments for the adversarial evaluation of our trained models, as well as checks against gradient masking.
Finally, we include the proof of Theorem~\ref{thm:uat_ft_gaussian} in Appendix~\ref{sec:proof_gaussian_model}.

\section{Experimental Details}
\label{sec:experimental_details}

\subsection{Implementation notes}
\label{sec:implementation_notes}
\textbf{Model architecture.}
For all experiments, we use variants of wide residual networks (WRNs)~\cite{zagoruyko2016wide}.
In Section~\ref{subsec:exp_control}, we use a WRN of width 2 and depth 28 for SVHN and a WRN of width 8 and depth 28 for \cifar{}.
We explore increasing the depth of the network (while keeping width to be 8) to 34, 70 and 106 in Section~\ref{subsec:exp_80m}.

\textbf{Data preprocessing.}
We use standard data augmentation techniques for images.
For \cifar{}, 4-pixel padding is used before performing random crops of size 32x32 and random left-right flip.
For SVHN, 4-pixel padding is also employed before random crops of size 32x32 followed by random color distortions.

\textbf{Pseudocode.} 
We provide pseudocode for our particular implementations of each UAT variant. 
To simplify notation, when writing $(x, y) \sim \unsupdata$, the target $y$ is always the fixed target pseudo-label (which is unused in \uatot).
Recall that these pseudo labels are obtained from a model trained on $\supdata$ alone.
\empadvloss and \empuatloss are the empirical estimates of the robust loss from \cite{madry18iclr} (as in \uatft) and $\mathcal{L}^{\text{OT}}$ respectively, as defined in the next section.

\begin{algorithm}
\caption{\uatot{} update}
\label{alg:uatot}
\begin{algorithmic}
\STATE {\bfseries Input:} Weight hyperparameter $\lambda$, batch sizes $b_s$ and $b_u$
\STATE Sample $b_s$ labeled examples $(\bmx_s, \bmy_s) \sim \supdata$ and $b_u$ unlabeled examples $(\bmx_u, \bmy_u) \sim \unsupdata$
\STATE Compute loss $L = \empadvloss(\bmx_s, \bmy_s; \theta) + \lambda (\frac{b_s}{b_u}) \empuatloss(\bmx_u, \bmy_u; \theta)$
\STATE Update with gradient $g = \nabla_\theta L$

\end{algorithmic}
\end{algorithm}

\begin{algorithm}
\caption{\uatft{} update}
\label{alg:uatft}
\begin{algorithmic}
\STATE {\bfseries Input:} Batch sizes $b_s$ and $b_u$
\STATE Sample $b_s$ labeled examples $(\bmx_s, \bmy_s) \sim \supdata$ and $b_u$ unlabeled examples $(\bmx_u, \bmy_u) \sim \unsupdata$
\STATE Merge $\bmx = [\bmx_s; \bmx_u]$; $\bmy = [\bmy_s; \bmy_u]$
\STATE Compute loss $L = \empadvloss(\bmx, \bmy; \theta)$
\STATE Update with gradient $g = \nabla_\theta L$
\end{algorithmic}
\end{algorithm}

\begin{algorithm}
\caption{\uatpp{} update}
\label{alg:uatpp}
\begin{algorithmic}
\STATE {\bfseries Input:} Weight hyperparameter $\lambda$, batch size $b_s$ and $b_u$
\STATE Sample $b_s$ labeled examples $(\bmx_s, \bmy_s) \sim \supdata$ and $b_u$ unlabeled examples $(\bmx_u, \bmy_u) \sim \unsupdata$
\STATE Merge $\bmx = [\bmx_s; \bmx_u]$; $\bmy = [\bmy_s; \bmy_u]$
\STATE Compute loss $L = \empadvloss(\bmx, \bmy; \theta) + \lambda \empuatloss(\bmx, \bmy; \theta)$
\STATE Update with gradient $g = \nabla_\theta L$
\end{algorithmic}
\end{algorithm}

\textbf{Loss implementations.} 
In the above, the empirical estimates of the losses are defined as follows:
\begin{align*}
\hat{\mathcal{L}}^{adv}(\bmx, \bmy; \theta) &= 
\frac{1}{|(\bmx,\bmy)|}	\sum_{i=1}^{|(\bmx,\bmy)|} \sup_{x'_i \in \normball(\bmx_i)} \xent(\bmy_i, p_\theta(\cdot | x'_i)) \\
\empuatloss(\bmx, \bmy; \theta) &=
	\frac{1}{|(\bmx,\bmy)|} \sum_{i=1}^{|(\bmx,\bmy)|} \sup_{x'_i \in \normball(\bmx_i)} \mathcal{D}(p_{\hat{\theta}}(.|\bmx_i),p_\theta(.|x'_i))
\end{align*}

We always approximate each maximization with 10-steps of PGD, as described in Appendix \ref{app:pgd_details}.
For $\empuatloss$, there are two implementational details of the attack which are not obvious from the pseudocode.
\begin{itemize}
\item When computing $\empuatloss$ in \uatot, we solve the adversarial optimization with a ``hard-label'' rather than ``soft-label'' attack.
That is, rather than maximizing the KL directly, we take the hard label $\hat{y} = \argmax_y p_\theta(y|x)$ and run a PGD attack using $\hat{y}$ as the label.
We suspect this is because when $p_\theta(y|x)$ is not completely one-hot, we are maximizing a convex objective.
There are (at least) two issues due to this. 
First, the gradient of the KL, evaluated at $x$, is 0, since $p(y|x)$ is the global minimum. 
Second, if the random initialization $x'$ causes $p(\hat{y}|x') > p(\hat{y}|x)$, then the gradient will encourage increasing $p(\hat{y}|x')$, rather than decreasing it. 
While previous work in \cite{Zhang19Trades} finds that initializing to a random perturbation of $x$ allows PGD to effectively maximize the KL, using hard labels worked better in our experiments.
\item When computing $\empuatloss$ in \uatpp, we reuse the adversarial example computed to maximize $\hat{\mathcal{L}}^{adv}$, for computational efficiency.
\end{itemize}
In both cases, the model loss is still a KL divergence.

\noindent
\textbf{Loss weights.}
For all experiments, we used $\lambda=5$ for both \uatot{} and \uatpp.
We fixed these values based on early experiments.

\noindent
\textbf{Differences with distribution shift.}
For our off-distribution experiment in Section \ref{subsec:exp_80m}, we make two changes to accomodate distribution shift.
First, we compute the loss on the labeled and unlabeled examples using separate forward passes through the network.
Without batch norm, this has no effect on the computation.
With batch norm, we observe this helps slightly, perhaps because the off-distribution unlabeled data degrades the local batch statistics.
Second, we downweight the loss of unlabeled examples by a factor of $b_s / b_u$ (this corresponds to averaging the loss across examples within each separate batch, rather than summing).
We suspect that while label noise has little effect, distribution shift degrades performance, and so focusing more on in-distribution examples is helpful.

\noindent
\textbf{Batch sizes.}
Throughout Section \ref{subsec:exp_control}, we use batch sizes proportional to the dataset size. In other words, given labeled and unlabeled datasets $\supdata$ and $\unsupdata$, for a total batch size of $B=128$, we use $b_s = BN / (N + M)$ and $b_u = BM / (N + M)$. 
This ensures we perform an equal number of passes through each dataset, and helps avoid overfitting the (small) labeled dataset.
For our high data regime experiment in Section \ref{subsec:exp_80m}, we use fixed $b_s=512$ and $b_u=4096$.

\noindent
\textbf{Optimization.}
We use the SGD with momentum optimizer for all training jobs.
For all training jobs, we use weight decay ($L_2$ regularization) weighted by $5\times 10^{-4}$.

In Section~\ref{sec:control_exp_analysis}, we use a total batch size of 128, for $25K$ steps with an initial learning rate of $0.2$ which is decreased by a factor 10 at iterations $15K$, $18K$ and $20K$.
This schedule was fine-tuned using the validation set.

In Section~\ref{sec:real_world_dist_shift}, we use the same learning rate schedule of~\cite{Zhang19Trades}'s public repository is  used for \cifar{}, \ie initiate with 0.2, then dividing by 10 after 15K and 22K steps.

\subsection{Negative results and observations}
We note several observations we made while running experiments.
Note that these are much less carefully examined than the results reported in the main paper.
Indeed, we suspect that some of these observations may be specific to our specific training setup.
However, we include these as we believe they may nonetheless be helpful for future researchers:
\begin{itemize}
\item We merge (pseudo)-labeled and unlabeled batches when possible, since this improved robust accuracy, particularly in the small-data regime (when using smaller batch sizes)
We suspect this is due to more reliable local batch statistics in batch normalization, since with small batch sizes, the labeled batch can be quite small.
\item We found learning rate schedules to be particularly important. For example, in Table \ref{tab:uat_largeScale}, the primary difference between our reimplementation of adversarial training, and the implementation in \cite{madry18iclr} is we use the significantly shorter schedule from \cite{Zhang19Trades}, which improves robust accuracy by \texttildelow3\%. 
\item With our current choices for $\lambda$, we notice a strong regularizing effect from $\empuatloss$. While \uatft{} reaches nearly 100\% train robust accuracy by the end of training, the same schedule used with \uatot{} or \uatpp{} stays below 80\% train robust accuracy.
\item When using \empuatloss on unsupervised data, we noticed it was necessary to either use fixed parameters $\hat{\theta}$ for computing the target prediction, or to also use fixed targets on the unlabeled data.
When both these are missing, the model would learn to predict a uniform class distribution for images in the unsupervised dataset.
While such models achieve high training accuracy, test set accuracy stays close to random, even on unperturbed images.
On labeled data, both versions with and without backpropagating into target predictions work fine, as reported in \cite{Zhang19Trades}.
\item We tried combining UAT with the recent semi-supervised UDA method~\cite{Xie19UAD} but did not see significant improvements in adversarial accuracy.
\item We experimented with alternative normalization strategies such as InstanceNorm \cite{ulyanov2016instance} and GroupNorm \cite{wu2018group}, but observed slightly worse results.
\end{itemize} 

\subsection{Projected gradient descent (PGD) details}
\label{app:pgd_details}

We provide additional details on our untargeted PGD attack, which we use for adversarial training.
The untargeted attack optimizes the margin loss objective proposed in \cite{carlini2017towards}
$$J_{\theta}^\mathrm{adv}(x) = Z(x, \theta)_y - \max_{i \neq y} Z(x, \theta)_i \: ,$$
where $Z(x, \theta)_i$ denotes the logit for class $i$, predicted by model $\theta$ on input $x$, and $y$ denotes the true label.
The margin loss is negative if and only if $x$ is misclassified.

We optimize this objective with projected gradient descent \cite{kurakin17advScale, madry18iclr} using the Adam optimizer \cite{kingma2014adam}.
Consistent with prior work (e.g. \cite{carlini2017towards, liu2016delving}), we find these modifications to improve attack convergence speed when compared to using the negative cross-entropy loss or vanilla gradient updates.
During training, we perform 10 steps of optimization, as in \cite{Zhang19Trades}.

\section{Implementation note on VAT baseline implementation}
\label{app:vat}

Recall that for UAT-OT, we use the loss introduced in~\cite{Miyato17VAT} and also used in \cite{Zhang19Trades}
\begin{equation}
\label{eq:app_uat_ot}
\lossot(\theta) = \mathev_{x\sim\unsupdist} \sup_{x'\in\mathcal{N}_\epsilon(x)}\mathcal{D}(p_{\hat{\theta}}(.|x),p_\theta(.|x')),
\end{equation}
where $\mathcal{D}$ is the Kullback-Leibler divergence, 
and $\hat{\theta}$ indicates a fixed copy of the parameters $\theta$ in order to stop the gradients from propagating.

In practice, computing the optimal perturbation $x'$ (in $\sup_{x'\in\mathcal{N}_\epsilon(x)}$) cannot be achieved in closed form, hence approximations have to be proposed.
In~\cite{Miyato17VAT}, the authors start from a second order Taylor approximation of $\mathcal{D}(p_{\hat{\theta}}(.|x),p_\theta(.|x'))$, then use a combination of finite differences to efficiently approximate the Hessian and power iteration method to find an estimate of the optimal perturbation.
Their end goal being \textit{standard generalization}, they observe that only one iteration of the power method is sufficient for good performance.
This specific procedure (Hessian approximation and eigenvector estimation though power method) is specifically tailored to the case where the chosen specification for the adversary corresponds to the $L_2$ ball: $\normball(x) = \{x' : \norm{x' - x}_2 \leq \epsilon \}$,  whereas we focus on the $L_\infty$ ball type of constraints.

For fair comparison, we adapt their algorithm to the $L_\infty$ ball by replacing the previously mentioned procedure by one step of FGSM with learning rate equal to $\epsilon$.

UAT-OT mainly differs from VAT in the fact that we instead use 10 steps of PGD to estimate $x'$.
Looking at Figure~\ref{fig:control_exp} in the main paper, UAT-OT clearly outperforms VAT, indicating that when it comes to adversarial generalization, one has to use a stronger adversary to generate the perturbation.

\section{Dataset Details}
\label{app:dataset_details}

We detail the procedure used to generate the \bigunsupn{} datasets, used for experiments in Section \ref{sec:real_world_dist_shift}.
To remove exact and near duplicate images, we follow \cite{douze2009evaluation} and use the GIST image descriptors \cite{oliva2001modeling} provided with the \tinydsfullname{} dataset \cite{80m}.
For every image, we compute the $L_2$ distance to its nearest neighbor in the \cifar{} test set, as measured in GIST feature space, and remove all images within distance $0.28$, totalling roughly one million images.
Following this, we manually checked for near duplicates, by random sampling and visualizing $L_2$ nearest neighbors.
In our initial version of this paper, we removed only exact duplicates, which produced robust accuracies roughly 0.5 - 2\% higher across all experiments.

For filtering, we use a WRN-28-10 \cite{zagoruyko2016wide} model trained on \cifar, which achieves 96.0\% accuracy on the \cifar{} test set.
Following the procedure used for the original \cifar{} dataset \cite{cifar10}, for each class in \cifar, we use hyponyms of the class name, based on the Wordnet hierarchy \cite{miller1995wordnet}.
This leaves roughly 2 million images remaining, though the class distribution is highly non-uniform (of these, 1 million correspond to the ``dog'' class, while just 50000 correspond to ``deer'' or ``frog'').
For each image, we record the probability assigned by the pretrained model to the associated class.
We restrictto images with associated probability over $0.5$, and take the top $N/10$ images per class.
For the \bigunsup{500} dataset, some classes contain less than $50$K such examples.
In this case, we randomly duplicate examples, to maintain class balance.
Finally, we note that although we use a simple procedure here, and did not experiment with alternatives, we believe developing better procedures to exploit uncurated data is an important and underexplored research direction.

\section{Code Release}
\label{app:code_release}

Example usage of our best performing model (WRN-106 trained with \uatpp~on 80m@200K) as well as all the 80m@N datasets we used to train our models can be found on our github repository.\footnote{\url{https://github.com/deepmind/deepmind-research/tree/master/unsupervised_adversarial_training}}

\section{Adversarial Evaluation Details}
\label{app:adv_eval_details}

\subsection{Multitargeted attack evaluation}
\label{sec:multitargeted_attack}

We evaluate using the \multitargeted attack proposed in~\cite{gowal2019alternative}, which we have found to be significantly stronger than commonly used PGD attacks, at the cost of increased computation.
The \multitargeted{} implementation is equivalent to the untargeted attack described above, but rather than performing a single optimization, instead runs a targeted attack against each possible class, and returns the image which minimizes the untargeted loss.
The targeted margin loss is
\[
J_{\theta}^\mathrm{adv}(x) = Z(x, \theta)_y - Z(x, \theta)_t \: ,
\]
where $t$ is the target class.
We use 200 steps with 20 random restarts for each class.

We have found the \multitargeted{} attack to reliably outperform untargeted PGD attacks, which in turn reliably outperform \fgsm. In ensuring the strongest results for untargeted PGD, we found that using 200 steps with 20 random restarts slightly outperforms 100 steps with 1000 random restarts, and hence report results using the former.
For example, for our strongest model, trained on the \bigunsup{200} dataset,
the \fgsm adversarial accuracy is 63.65\%, the untargeted PGD adversarial accuracy is 61.10\%,  and the \multitargeted{} adversarial accuracy is 56.30\%. 

\subsection{SPSA evaluation}
\label{app:spsa_eval}
As an additional check against gradient masking \cite{athalye2018obfuscated, uesato2018adversarial}, we run a gradient-free attack, SPSA \cite{uesato2018adversarial}, against our strongest model, \uatpp{} trained on the \ \bigunsup{200} unsupervised dataset.
For SPSA, we use a batch size of 8192 with 40 iterations.
We obtain 64.9\% adversarial accuracy, similar to the 61.1\% obtained by untargeted PGD.
We further observe in Figure \ref{fig:spsa_scatter} that SPSA and PGD reliably converge to similar loss values, and that SPSA rarely outperforms PGD.
This provides additional evidence that the model's strong performance is not due to gradient masking. %

\begin{figure}[ht]
	\begin{center}
		
		\begin{tikzpicture}
		\begin{axis}[
		scatter, axis x line=center, axis y line=center,
		width=0.47\textwidth, height=0.4\textwidth,
		xlabel={$J_{\theta}^\mathrm{adv}(x)$ using PGD},
		ylabel={$J_{\theta}^\mathrm{adv}(x)$ using SPSA},
		x label style={at={(axis description cs:0.5,-0.1)},anchor=north},
		y label style={at={(axis description cs:-0.1,.5)},rotate=90,anchor=south},
		domain=-15:10,
		]
		
		\addplot[only marks, mark size=1pt, opacity=0.7] table [x index=0, y index=1, scatter src=-\thisrowno{2}, col sep=comma] {comparison_scatter.csv};
		\addplot[black, domain=-14.5:10, mark=none, scatter src=explicit] coordinates {(-7, -7)[0] (13, 13)[0]};
		\end{axis}
		\end{tikzpicture}
		
		\caption{\textbf{Analysis of gradient masking with SPSA}: We compare the final values of the margin loss across different images, between PGD and SPSA. Each point represents a single image, which is misclassified when $J_\theta^\mathrm{adv}(x) < 0$. 
			Overall, we find that SPSA and PGD converge to similarly adversarial perturbations (points close to the line $y=x$).
			We observe relatively few images where SPSA outperformed PGD (below the line, shown in red). The upward bend in the dots to the left of $-3$ are an artifact due to the fact that we terminate the SPSA attack early, once we find any image with margin loss below $-3$.}
		\label{fig:spsa_scatter}
	\end{center}
\end{figure}
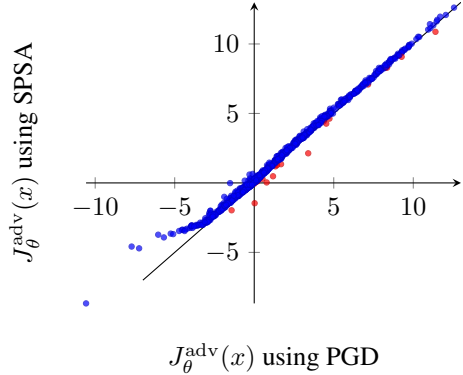

\subsection{Loss landscape analysis}
\label{sec:loss_landscape}
As another check against gradient masking, we look at the adversarial loss landscape from our strongest model.
We examine the loss surfaces for four images where the \texttt{MultiTargeted} attack succeeded, but untargeted PGD attack, with 200 steps and 20 restarts, did not. %
Figure~\ref{fig:loss_landscape} shows the untargeted adversarial loss (optimized by PGD) around the nominal image from \cifar.
In these loss landscapes, we vary the input along a linear space defined by the worse perturbations found by PGD and a random direction.
The $u$ and $v$ axes represent the magnitude of the perturbation added in each of these directions respectively and the $z$ axis represents the loss.
For figures on the right hand side of Figure~\ref{fig:loss_landscape}, they show a top-view of the loss landscape and indicates that a large portion of $L_\infty$~ball around the nominal image pushes the PGD solution towards the right (rather than the bottom).
We observe that the loss landscape is rather smooth, which provides (weak) additional evidence that the strong performance is not due to gradient masking.

\begin{figure}[h]
	\centering
	\begin{subfigure}{0.4\textwidth}
		\includegraphics[width=\linewidth, trim=2.2cm 0cm 0cm 1.3cm, clip]{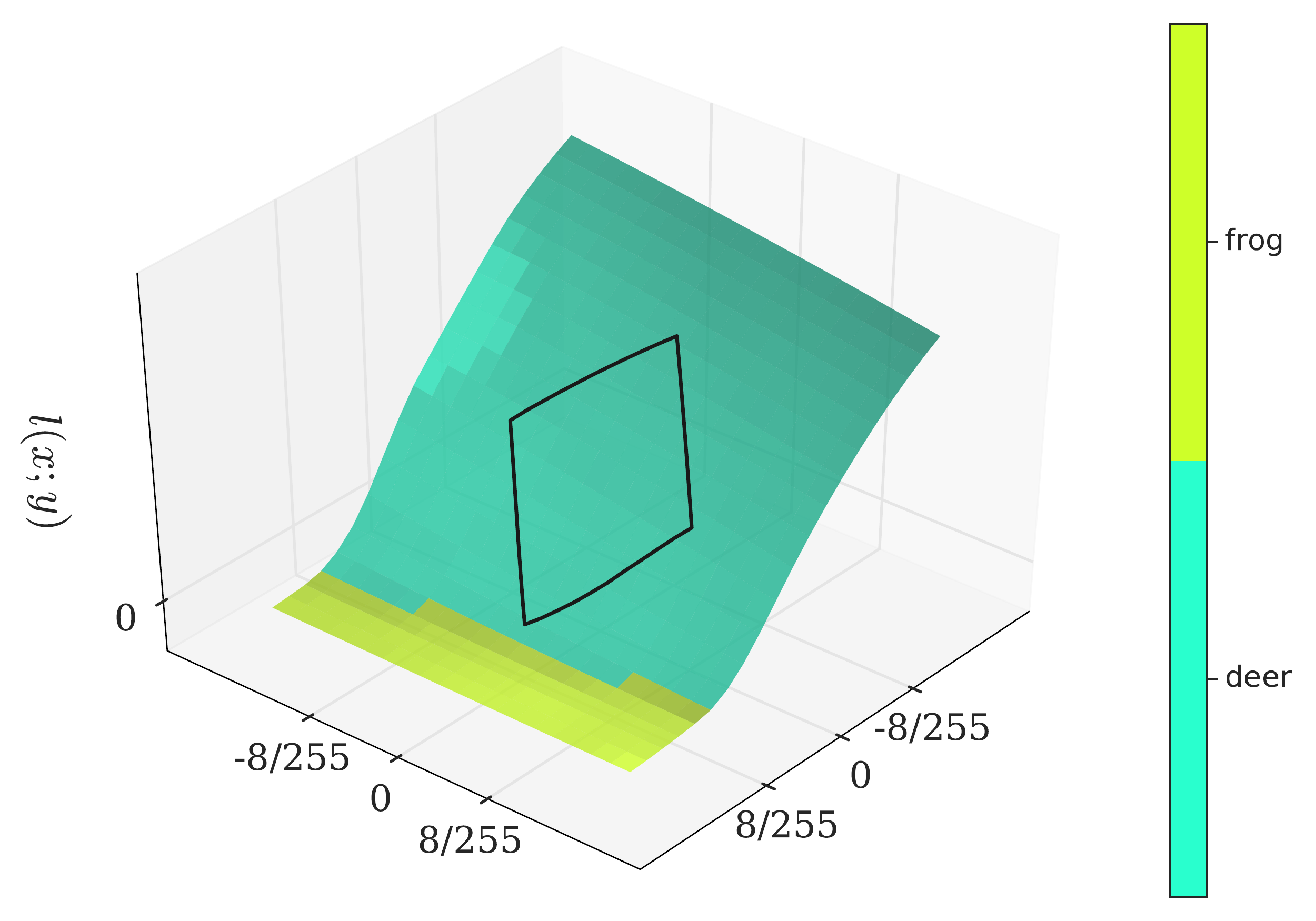}
		{\label{fig:loss_landscape_3d_9936}}
	\end{subfigure} \hspace{.5cm}
	\begin{subfigure}{0.42\textwidth}
		\includegraphics[width=\linewidth]{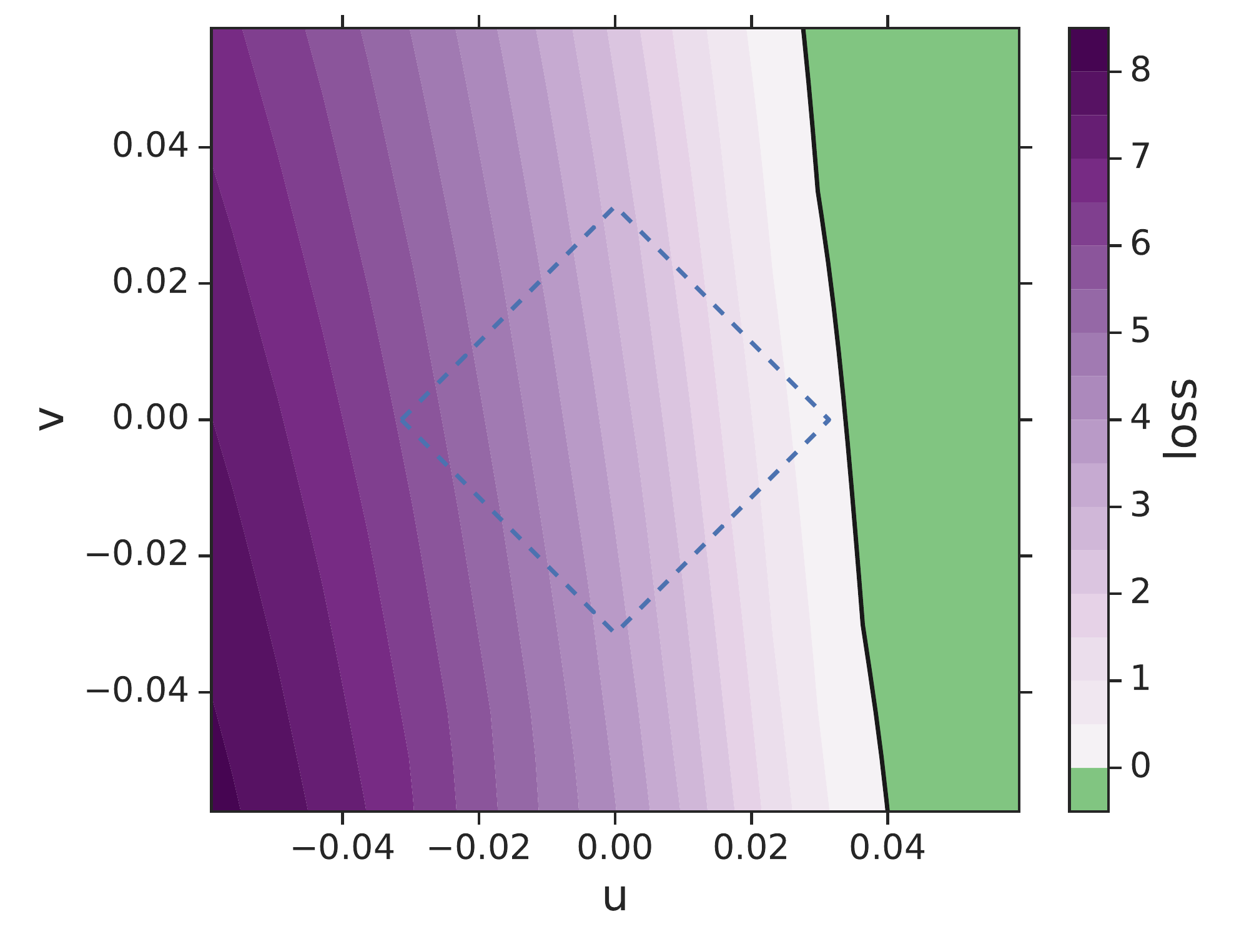}
		{\label{fig:loss_landscape_2d_9936}}
	\end{subfigure}
	\begin{subfigure}{0.4\textwidth}
		\includegraphics[width=\linewidth, trim=2.2cm 0cm 0cm 1.3cm, clip]{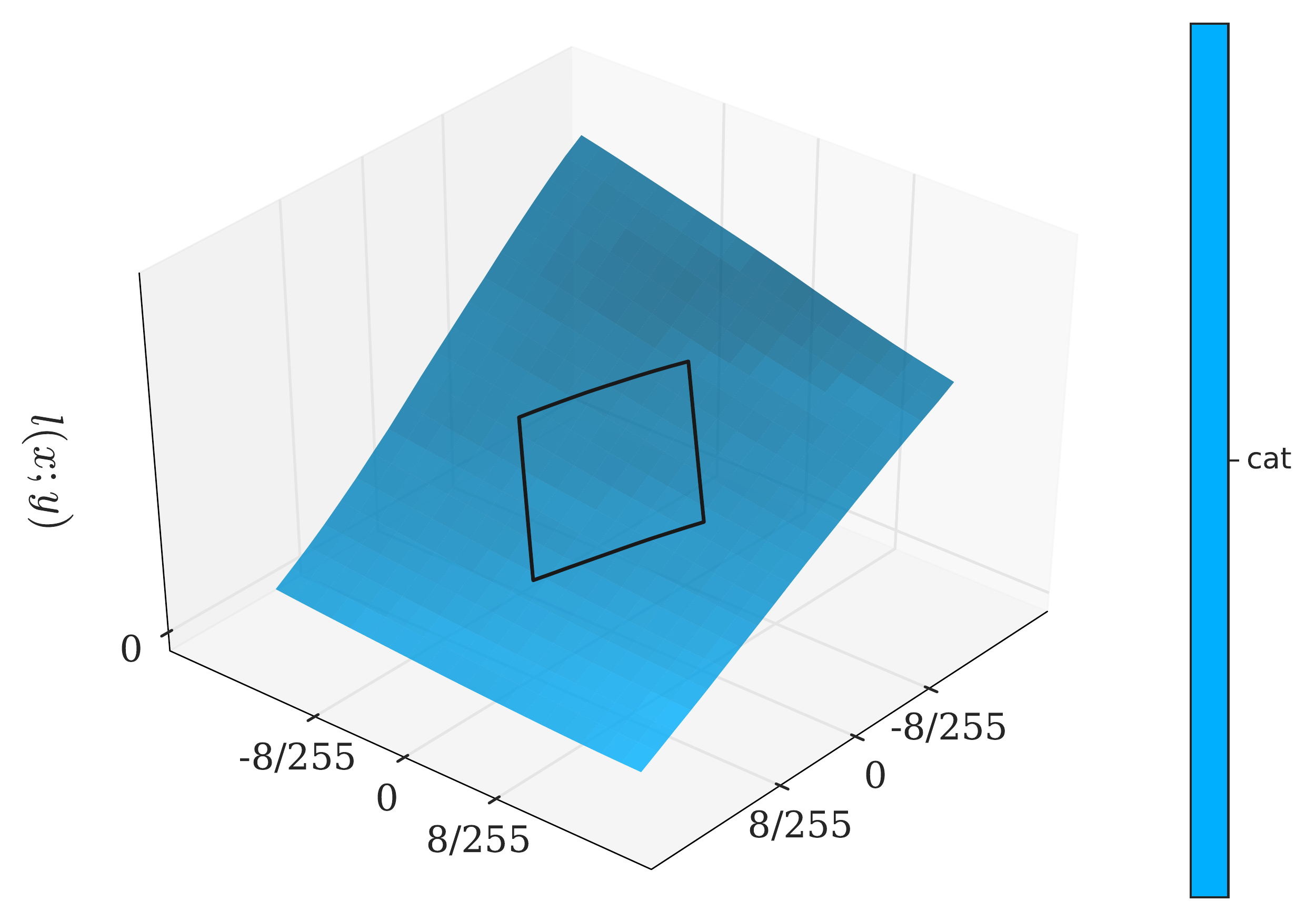}
		{\label{fig:loss_landscape_3d_9937}}
	\end{subfigure} \hspace{.5cm}
	\begin{subfigure}{0.42\textwidth}
		\includegraphics[width=\linewidth]{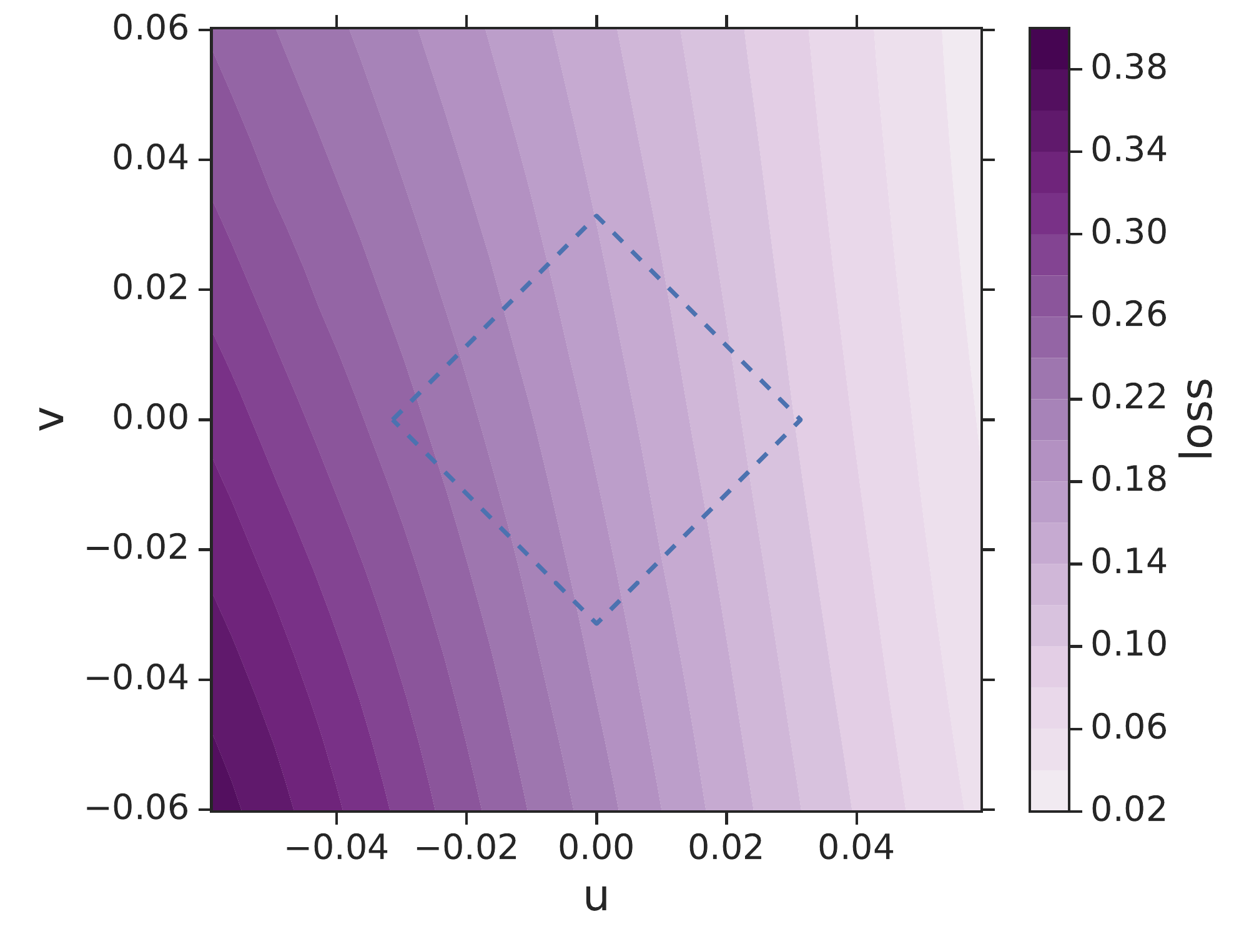}
		{\label{fig:loss_landscape_2d_9937}}
	\end{subfigure}\vspace{-5mm}
	\begin{subfigure}{0.4\textwidth}
		\includegraphics[width=\linewidth, trim=2.2cm 0cm 0cm 1.3cm, clip]{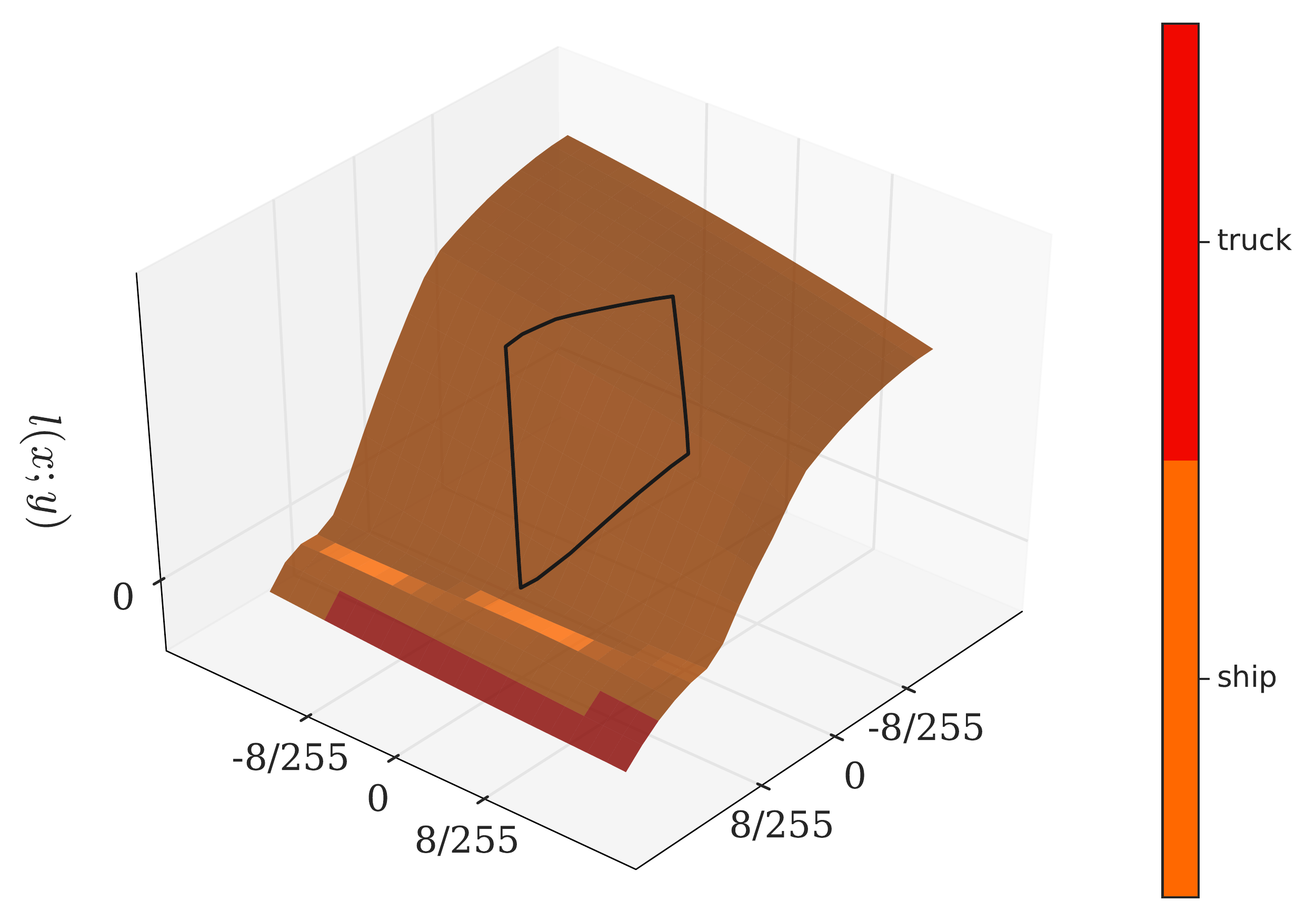}
		{\label{fig:loss_landscape_3d_9938}}
	\end{subfigure} \hspace{.5cm}
	\begin{subfigure}{0.42\textwidth}
		\includegraphics[width=\linewidth]{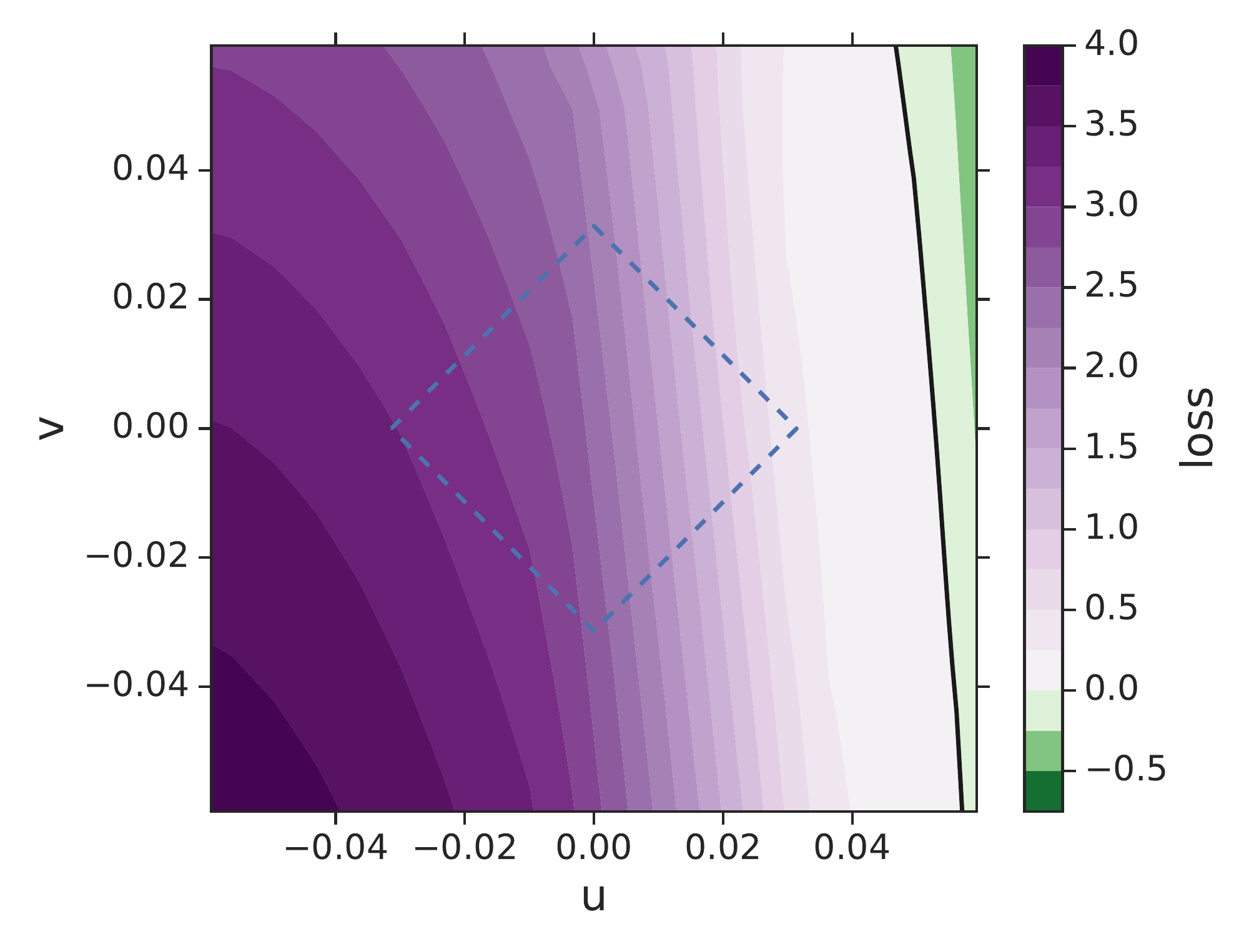}
		{\label{fig:loss_landscape_2d_9938}}
	\end{subfigure}\vspace{-5mm}
	\begin{subfigure}{0.4\textwidth}
		\includegraphics[width=\linewidth, trim=2.2cm 0cm 0cm 1.3cm, clip]{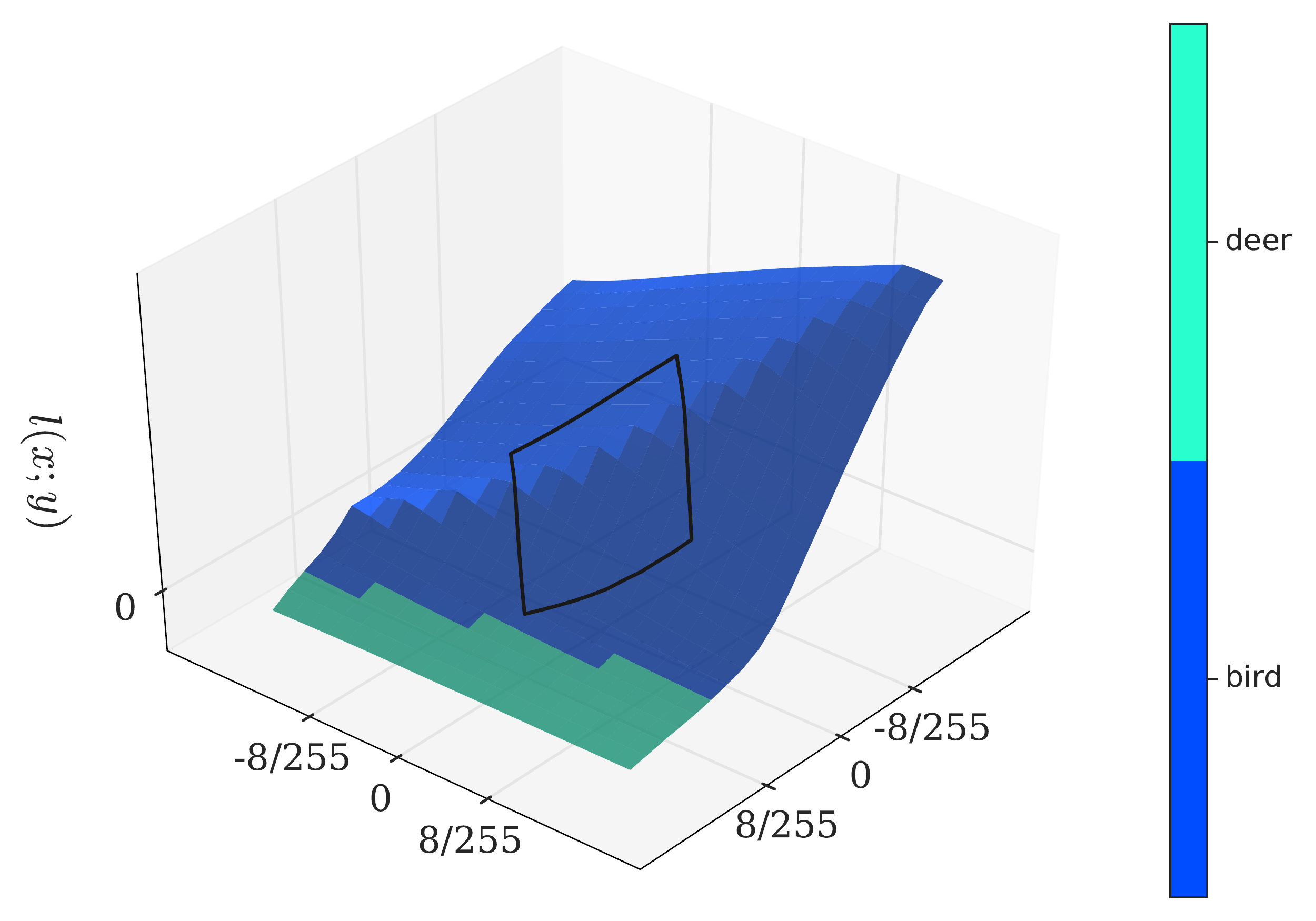}
		{\label{fig:loss_landscape_3d_9961}}
	\end{subfigure} \hspace{.5cm}
	\begin{subfigure}{0.42\textwidth}
		\includegraphics[width=\linewidth]{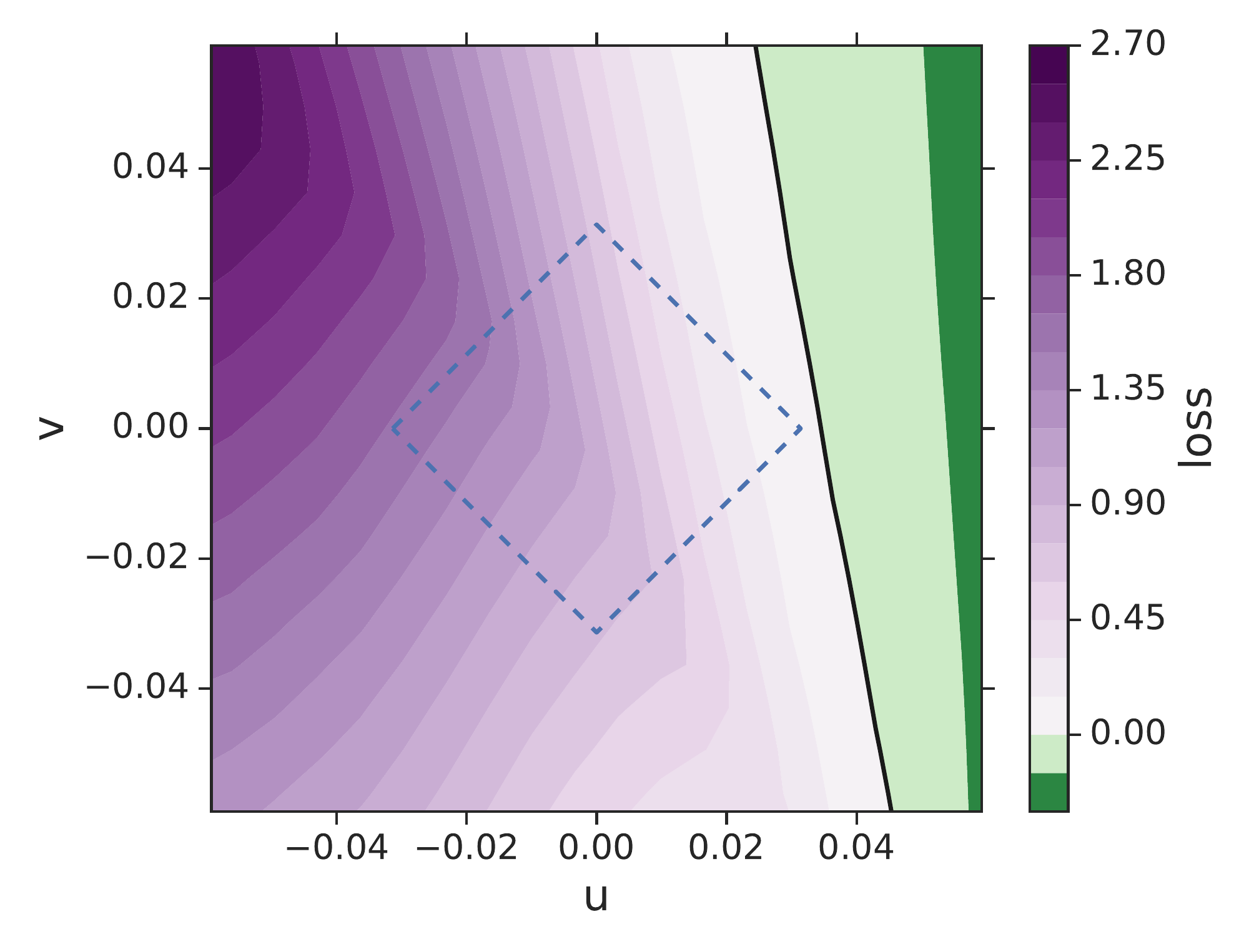}
		{\label{fig:loss_landscape_2d_9961}}
	\end{subfigure}\vspace{-5mm}

	\caption{Adversarial loss landscapes around the nominal images.
		It is generated by varying the input to the model, starting from the original input image toward either the worst attack found using PGD ($u$ direction) or the one found using a random direction ($v$ direction). For the figures on the left hand side, the z axis represents the loss. For both panels, the diamond-shape represents the projected $L_\infty$~ball of size $\epsilon = 8/255$ around the nominal image.}
	\label{fig:loss_landscape}
\end{figure}

\subsection{Attack convergence analysis}
\label{sec:attack_convergence}
As another check against gradient masking, we analyzed the convergence of PGD.
Figure \ref{fig:attack_convergence} shows convergence of untargeted PGD across different random restarts for our strongest model, though we also observe similar patterns across other models.
We observed that on randomly selected images, PGD quickly converges, and that the final loss values across random restarts are tightly clustered, indicating PGD likely converges to near-optimal perturbations.
Figure \ref{fig:attack_convergence_rand} shows randomly selected images, and is consistent with what we observe across images where \multitargeted{} and PGD agree.
In the fraction of cases where \multitargeted{} succeeded but PGD did not, we can find evidence of gradient masking on some images through random restarts.
In Figure \ref{fig:attack_convergence_mislabeled}, there are two images where the final loss varies across different random restarts.

\begin{figure}[h]
	\centering
	\begin{subfigure}{\textwidth}
		\includegraphics[width=\linewidth]{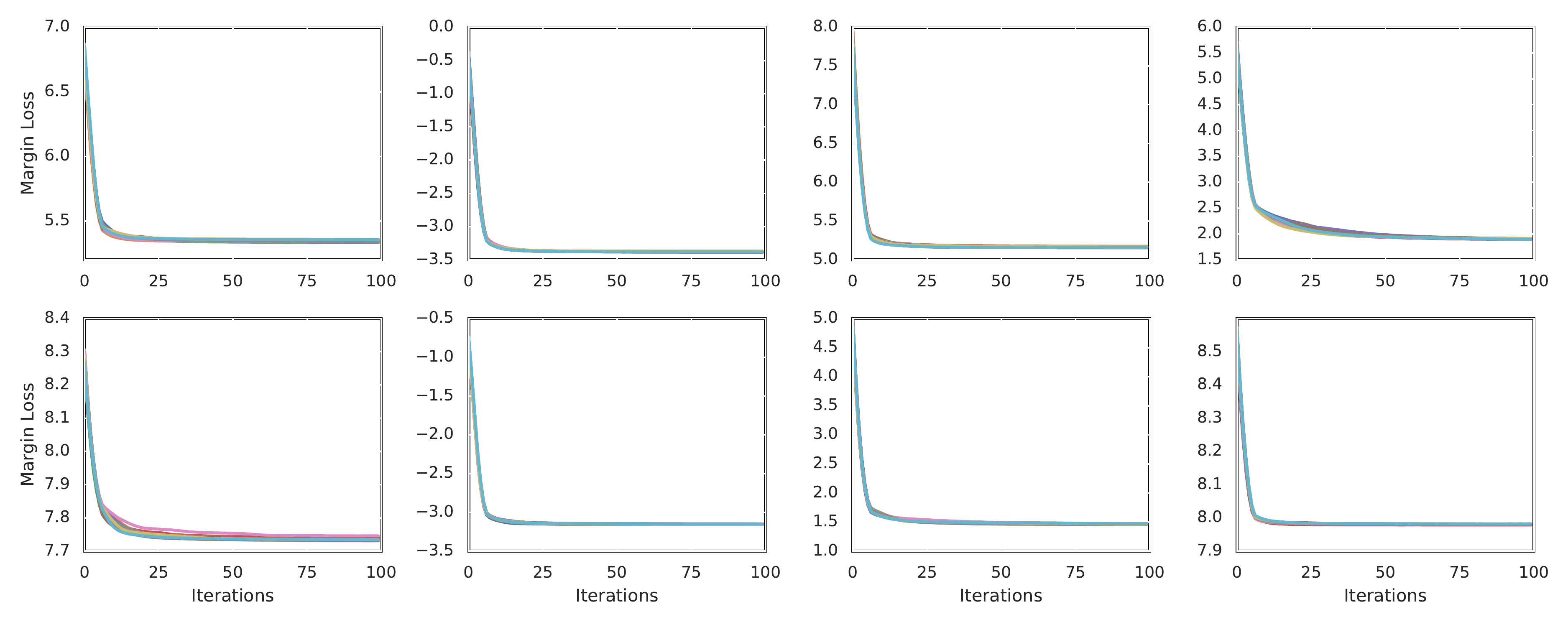}
		
		{\caption{Convergence on randomly sampled images}\label{fig:attack_convergence_rand}}
	\end{subfigure}
	\begin{subfigure}{\textwidth}
		\includegraphics[width=\linewidth]{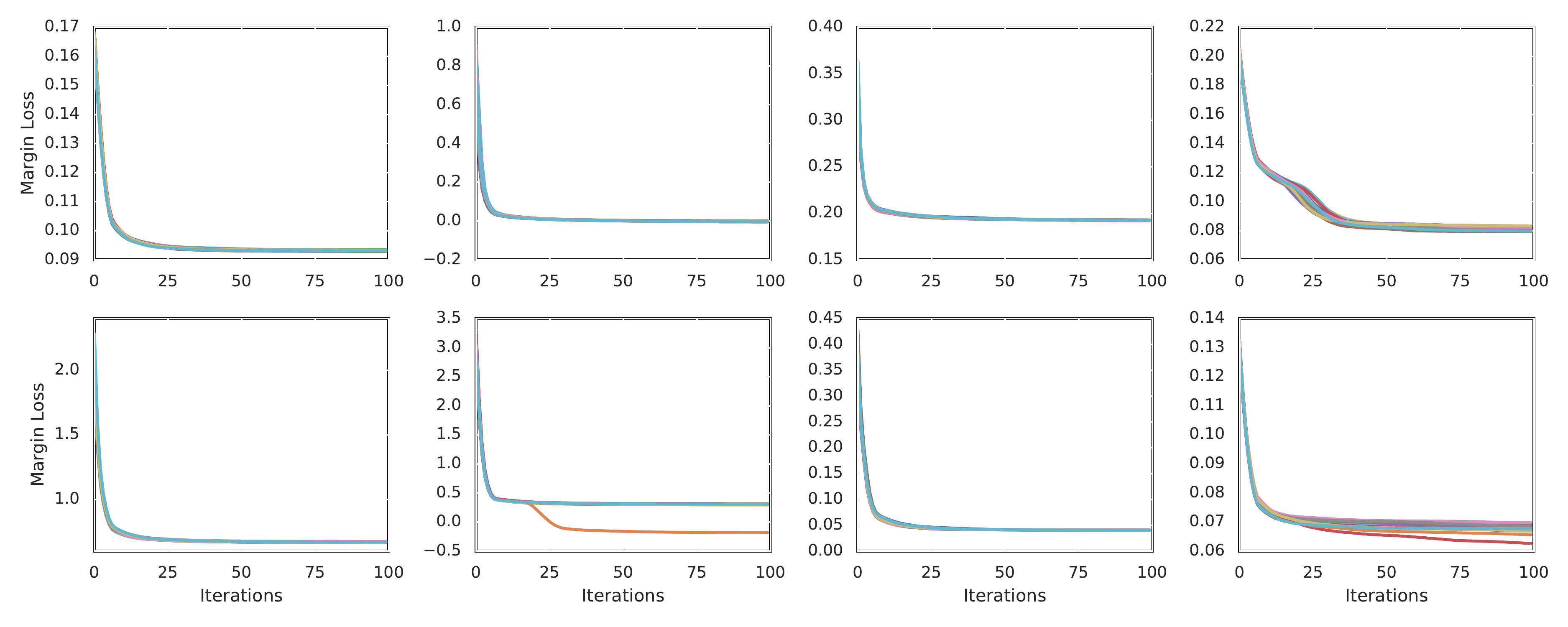}
		{\caption{Convergence on images where \multitargeted{} attack succeeded but untargeted PGD did not} \label{fig:attack_convergence_mislabeled}}
	\end{subfigure}
	
	\caption{\textbf{Convergence of PGD.} Each plot shows the convergence of the adversary loss on the same image, across 20 random restarts.
		On randomly sampled images (top), the loss converges to tightly clustered values.
		On images where PGD did not find optimal perturbations (bottom), we observe variation in perturbation strength across different restarts for two images in the bottom row.
	}
	\label{fig:attack_convergence}
\end{figure}

\section{Additional Experimental Results}
\label{sec:add_exp_results}
\textbf{$L_2$ robustness.}
We ran several short experiments to ensure our results hold for $L_2$ in addition to $L_\infty$ robustness. We use 4K labeled and 32K unlabeled examples.
On \cifar{} at $L_2$ radius $\epsilon = 0.87$, which encloses the $L_\infty$ $\epsilon=4/255$ ball, the purely supervised model achieves 32.7\% robust accuracy, the supervised oracle achieves 53.9\%, and UAT almost matches this, with 55.2\% robust accuracy. 
This represents a 21\% absolute gain from using unlabeled data, which captures over 90\% of the oracle improvement, without using additional labels.
We observe similar results for $\epsilon = 0.435$, which encloses the $L_\infty$ $\epsilon=2/255$ ball, of 47.3\% / 70.3\% / 66.3\% for the purely supervised baseline, supervised oracle, and UAT, respectively.

\textbf{Number of necessary labels.}
To study the minimum number of labels required while maintaining robustness, we also train \cifar{} models using fewer labels. 
In the body of the paper, we report that with 4K labels (and 32K unlabeled examples), UAT achieves 54.1\% robust accuracy, compared to 55.5\% for the supervised oracle which uses 36K labeled examples. 
We also trained models using 2K and 1K labels, which yield robust accuracies of 51.9\% and 47.7\% respectively. 
Thus, there is some loss of robustness -- with 4K labels, UAT almost exactly matches the performance of the supervised oracle, with only a 1.4\% gap, whereas with 2K and 1K labels, the gap is larger. 
However, even in this regime, UAT still achieves significant adversarial robustness.

\section{Proof of Theorem 1}
\label{sec:proof_gaussian_model}

\subsection{Preliminaries}

We first provide the following two concentration inequalities which we will use to bound our main quantities of interest.

\begin{lemma}[Concentration of $\chi$-squared distribution]
Let $X \sim \mathcal{N}(0, \sigma^2 I_n)$. Then, provided $\alpha^2 > 2n \sigma^2$,
\[
\PP(\| X \|^2 \geq \alpha^2) \leq e^{-\alpha^2/(20 \sigma^2)}.
\]
\end{lemma}
\begin{proof}
The result follows from application of Lemma 1 from \cite{laurent2000adaptive}.
\end{proof}

\begin{lemma}
\label{lem:l1norm}
Let $X\sim\mathcal{N}(0, \sigma^2 I_m)$ in $\RR^m$.
Then
\[
\PP(\smfrac{1}{m}\|X\|_1 \ge a) \le 2^m \exp -\frac{ma^2}{2\sigma^2}
\]
\end{lemma}
\begin{proof}
Forming the Chernoff bound with $t=\frac{ma}{\sigma^2}$, we have:
\begin{align*}
\PP(\smfrac{1}{m}\|X\|_1 \ge a) &\le \exp -\frac{ma^2}{\sigma^2} \mathev[\exp \frac{a}{\sigma^2} \|X\|_1] \\
&= \exp -\frac{ma^2}{\sigma^2} \left( \mathev[\exp \frac{a}{\sigma^2} |X_1|] \right)^m \\
&= \exp -\frac{ma^2}{\sigma^2} \left( \exp \frac{a^2}{2\sigma^2} (1+\mathrm{erf}\smfrac{a}{\sigma\sqrt{2}}) \right)^m \\
&= \exp -\frac{ma^2}{\sigma^2} \exp \frac{ma^2}{2\sigma^2} \left( (1+\mathrm{erf}\smfrac{a}{\sigma\sqrt{2}}) \right)^m \\
&\le 2^m \exp -\frac{ma^2}{2\sigma^2}
\end{align*}
\end{proof}

\subsection{Main Proof}
To bound the robustness, there are two main quantities of interest. 
First, we need to bound the norm of $\vzbar = \frac{1}{m} \sum_{i=1}^{m} \yhi x_i$, which controls the smoothness of the classifier (Lemma \ref{lem:zbar_norm}).
Second, we need to bound the inner product $\braket{\vzbar}{\theta^*}$, which controls how well the classifier fits the data (Lemma \ref{lem:zbar_inner_product}).

The main difficulty is that $\sum_{i=1}^{m} \yhi x_i$ is not Gaussian distributed.
In particular, while $\sum_{i=1}^{m} y_i x_i$ follows a Gaussian distribution, our quantity of interest does not, due to the dependence of $\yhi$ on $x_i$. \\ %

\begin{lemma}
\label{lem:zbar_norm}
Given a $(\thetastar, \sigma)$ Gaussian model in $\RR^d$, let $h: \RR^d \rightarrow \{-1, +1\}$ be any classifier.
If $\vzbar = \frac{1}{m} \sum_{i=1}^{m} \yhi x_i$ is the sample mean vector of $m$ i.i.d.\ samples based on predicted classes $\yhi=h(x_i)$, then we have
\[
\PP\left( \| \vzbar \|_2 \geq (1+c) \| \theta^* \|_2 + 2 \sigma \sqrt{\frac{d}{m}} \right) \leq e^{-6\sqrt{d}/5},
\]
with $c = \frac{\sqrt{20} \sigma}{\|\theta^*\|} \sqrt{\frac{\sqrt{d}}{m} + \log 2}$. 
\end{lemma}
\begin{proof}
We have
\begin{align*}
\left\| \frac{1}{m} \sum_{i=1}^m \hat{y}_i x_i \right\|_2 = \left\| \frac{1}{m} \sum_{i=1}^m \hat{y}_i (y_i \theta^* + z_i ) \right\|_2 & \leq \frac{\| \theta^* \|}{m} \sum_{i=1}^m \hat{y}_i y_i + \left\| \frac{1}{m} \sum_{i=1}^m \hat{y}_i z_i \right\|_2 \\
& \leq \| \theta^* \| + \left\| \frac{1}{m} \sum_{i=1}^m \hat{y}_i z_i \right\|_2 \; ,
\end{align*}
where $z_i \sim \mathcal{N}(0, \sigma^2 I)$.
We therefore have
\begin{align*}
\PP \left( \left\| \frac{1}{m} \sum_{i=1}^m \hat{y}_i x_i \right\|_2 \geq t \right) & \leq \PP \left( \left\| \frac{1}{m} \sum_{i=1}^m \hat{y}_i z_i \right\|_2 \geq t - \| \theta^* \| \right) \\
          & \leq  \PP \left( \bigcup_{s_1 = \pm 1, \dots, s_m = \pm 1} \left\| \frac{1}{m} \sum_{i=1}^m s_i z_i \right\|_2 \geq t - \| \theta^* \| \right) \\
           & \leq \sum_{s} \PP \left( \frac{1}{m} \left\| \sum_{i=1}^m s_i z_i \right\|_2 \geq t - \| \theta^* \| \right)
\end{align*}
Observe that $\sum_{i=1}^m s_i z_i \sim \mathcal{N} (0, m \sigma^2)$. Now, using the concentration of measure result, whenever $t-\| \theta^* \| \ge \sqrt{\frac{2}{m}}\sigma$ we have:
\[
\PP \left( \left\| \sum_{i=1}^m s_i z_i \right\|^2_2 \geq m^2 (t - \| \theta^* \|)^2 \right) \leq e^{-m^2 (t - \| \theta \|)^2 / (20 m \sigma ^2) } = e^{-m (t - \| \theta \|)^2 / (20 \sigma ^2) } 
\]
Hence, we obtain
\[
\PP \left( \left\| \frac{1}{m} \sum_{i=1}^m \hat{y}_i x_i \right\|_2 \geq t \right) \leq 2^m e^{-m (t - \| \theta^* \|)^2 / (20 \sigma ^2)}.
\]
Let $t = (1+c) \| \theta^* \| + 2 \sigma \sqrt{\frac{d}{m}}$. Then, we have
\begin{align*}
\PP \left( \left\| \frac{1}{m} \sum_{i=1}^m \hat{y}_i x_i \right\|_2 \geq t \right) & \leq 2^m e^{-m \left( c \| \theta^* \| + 2 \sigma \sqrt{\frac{d}{m}} \right)^2 / (20 \sigma^2)} \\
& \leq 2^m e^{-m \left( c^2 \| \theta^* \|_2^2 / (20 \sigma^2) \right) } e^{-m 4 \sigma^2 \frac{d}{m} / (20 \sigma^2)} \\
& = 2^m e^{-m \left( c^2 \| \theta^* \|_2^2 / (20 \sigma^2) \right) } e^{-d/5}
\end{align*}
Now let $c = \frac{\sqrt{20} \sigma}{\| \theta^* \|_2} \sqrt{\frac{\sqrt{d}}{m} + \log 2}$. Then, the above probability is given by
\[
2^m e^{-m \left( \frac{\sqrt{d}}{m} + \log 2 \right)} e^{-d/5} = e^{-\sqrt{d}} e^{-d/5} \le e^{-\sqrt{d}} e^{-\sqrt{d}/5} = e^{-6\sqrt{d}/5}.
\]
\end{proof}

\begin{lemma}
\label{lem:zbar_inner_product}
Under the conditions of Lemma \ref{lem:zbar_norm}, let $p = \mathev[\I[h(x) = y]]$ denote the accuracy of classifier $h$. Then we have
\[
\PP\left( \braket{\vzbar}{\theta^*} \leq \tchern \| \theta^* \|^2 - \sqrt{2} \| \theta^* \| \sigma \sqrt{\frac{\log(1/\delta)}{m} + \log 2} \right) \leq (\tailprobchern)^{mp} + \delta.
\]
\end{lemma}
\begin{proof}
We write
\[
\PP\left( \braket{\hat{w}}{\theta^*} \leq t \right) = \PP\left( \braket{\frac{1}{m} \sum_{i=1}^m \hat{y}_i x_{i}}{\theta^*} \leq t \right) = \PP\left( \braket{\frac{1}{m} \sum_{i=1}^m \hat{y}_i (y_i \theta^* + z_{i})}{\theta^*} \leq t \right),
\]
where $z_i$ are $\mathcal{N}(0, \sigma^2 I)$.
The expression inside the probability is equal to
\[
\frac{\| \theta^* \|_2^2}{m} \sum_{i=1}^m \hat{y}_i y_i + \frac{1}{m} \sum_{i=1}^m \braket{\hat{y}_i z_i}{\theta^*}
\]
We bound this expression from below with
\[
\frac{\| \theta^* \|_2^2}{m} \sum_{i=1}^m \hat{y}_i y_i - \frac{1}{m} \sum_{i=1}^m |\braket{z_i}{\theta^*}|
\]
(That is, we consider the worst-case scenario where the random variables $\hat{y}_i$ are given by the negative of the sign of $\braket{z_i}{\theta^*}$).
We therefore get that
\begin{align*}
\PP\left( \braket{\hat{w}}{\theta^*} \leq t - t' \right) & \leq \PP \left( \frac{\| \theta^* \|_2^2}{m} \sum_{i=1}^m \hat{y}_i y_i - \frac{1}{m} \sum_{i=1}^m |\braket{z_i}{\theta^*}| \leq t - t' \right) \\
& \leq \PP \left( \frac{\| \theta^* \|_2^2}{m} \sum_{i=1}^m \hat{y}_i y_i \leq t \right) + \PP \left( \frac{1}{m} \sum_{i=1}^m |\braket{z_i}{\theta^*}| \geq t' \right) \\
& = \PP \left( \frac{2 \| \theta^* \|_2^2}{m} \sum_{i=1}^m \I[y_i = \hat{y}_i] - \| \theta^* \|_2^2 \leq t \right) + \PP \left( \frac{1}{m} \sum_{i=1}^m |\braket{z_i}{\theta^*}| \geq t' \right)
\end{align*}
We treat the first term.
Let $t = \tchern \| \theta^* \|^2$. The first probability term is hence given by
\[
\PP\left( \sum_{i=1}^m \I[y_i = \hat{y}_i] \leq \frac{7}{8} mp \right) \leq \exp \left( -\frac{mp}{2 \cdot 8^2} \right) \leq (\tailprobchern)^{mp}.
\]
using a Chernoff bound.%

We have the following concentration bound on the $\ell_1$ norm of the Gaussian vectors $U \sim \mathcal{N}(0, \| \theta^* \|^2 \sigma^2)$:
\[
\PP\left( \frac{1}{m} \| U \|_1 \geq t' \right) \leq 2^m \exp\left( -t'^2 m / (2 \| \theta^* \|^2 \sigma^2) \right),
\]
by applying Lemma \ref{lem:l1norm}. 
We set $t' = \sqrt{2} \| \theta^* \| \sigma \sqrt{\frac{\log(1/\delta)}{m} + \log 2}$, and when plugging in the above formula, we obtain
\[
\PP \left( \frac{1}{m} \sum_{i=1}^m |\braket{z_i}{\theta^*}| \geq t' \right) \leq \delta.
\]
Hence, we obtain the desired bound.
\end{proof}

We now use these results to achieve the final result in Theorem \ref{thm:uat_ft_gaussian}.
In what follows, we assume:
\begin{itemize}
    \item $(x_1, y_1), \ldots, (x_m, y_m)$ are drawn i.i.d.\ from a $(\thetastar, \sigma)$ Gaussian model in $\RR^d$ with mean norm $\|\thetastar\|_2 = \sqrt{d}$
    \item $h: \RR^d \rightarrow \{-1, +1\}$ is a base classifier with accuracy $p > \smfrac{3}{4}$, where $p = \mathev[\I[h(x) = y]]$
    \item $\vzbar \in \RR^d$ is the sample mean vector $\vzbar = \frac{1}{m} \sum_{i=1}^{m} \yhi x_i$, where $\yhi = h(x_i)$
    \item $\what \in \RR^d$ is the unit vector in the direction of $\vzbar$, i.e.,
    $\what = \sfrac{\vzbar}{\norm{\vzbar}_2}$
    \item $c$ denotes the constant in Lemma \ref{lem:zbar_norm}
\end{itemize}

\begin{lemma}
  \label{lem:normalized_inner_product}
  Under these assumptions,
  \[
    \prob\brackets*{ 
    \ip{\what, \thetastar} \, \leq \, \frac{ \tchern \sqrt{dm} - 
    \sqrt{d + 2m \sigma^2 \log 2}}{(1 + c)\sqrt{m} + 2\sigma} 
    }
  \]
  is bounded above by $\exp(-6\sqrt{d}/5) + (\tailprobchern)^{mp} + \exp(-d/2\sigma^2)$.
\end{lemma}
\begin{proof}
    By Lemma \ref{lem:zbar_inner_product}, we have
\begin{align*}
\prob \brackets*{ \ip{\vzbar, \theta^*} \geq \tchern \| \theta \|^2 - \sqrt{2} \| \theta \| \sigma \sqrt{\frac{\log(1/\delta)}{m} + \log 2} } \geq 1 - (\tailprobchern)^{mp} - \delta.
\end{align*}
Further, by Lemma \ref{lem:zbar_norm}, we have
\[
\PP\left( \| \hat{w} \|_2 \leq (1+c) \| \theta \| + 2 \sigma \sqrt{\frac{d}{m}} \right) \geq 1 - e^{-6\sqrt{d}/5},
\]
Conditioning on both events with $\delta = \exp(-d/2\sigma^2)$, the overall failure probability is bounded by $\exp(-6\sqrt{d}/5) + (\tailprobchern)^{mp} + \exp(-d/2\sigma^2)$.
Then, we have
\begin{align*}
\ip{\what, \thetastar} &= \frac{\ip{\vzbar, \thetastar}}{\norm{z}_2} \\
&\geq \frac{\tchern d - \sqrt{2d}\sigma\sqrt{\frac{\log \sfrac{1}{\delta}}{m} + \log 2}}
{(1 + c)\sqrt{d} + 2\sigma\sqrt{\frac{d}{m}}} \\
&= \frac{\tchern \sqrt{dm} - \sigma\sqrt{2}\sqrt{\log \sfrac{1}{\delta} + m\log 2}}
{(1 + c)\sqrt{m} + 2\sigma} \\
&= \frac{\tchern \sqrt{dm} - \sqrt{d + 2\sigma^2 m \log 2}}
{(1 + c)\sqrt{m} + 2\sigma} \\
\end{align*}

\end{proof}

For ease of reference, we provide a relevant lemma proved in \cite{Schmidt18moredata}.

\begin{lemma}[\cite{Schmidt18moredata}]
  \label{lem:schmidt_lemma_20}
Assume a $(\thetastar, \sigma)$-Gaussian model.
Let $p \geq 1$, $\eps \geq 0$ be robustness parameters, and let $\what$ be a unit vector such that $\ip{\what, \thetastar} \geq \eps \norm{\what}_p^*$., where $\norm{\cdot}_p^*$ is the dual norm of $\norm{\cdot}_p$.
Then the linear classifier $f_{\what}$ has $\ell_p^\eps$-robust classification error at most
\[
  \exp\parens*{-\frac{\parens{\ip{\what, \thetastar} - \eps \norm{\what}_p^*}^2}{2 \sigma^2}} \; .
\]
\end{lemma}

\begin{lemma}
	\label{lem:uatft_error}
  With probability at least $ 1 - [\exp(-6\sqrt{d}/5) + (\tailprobchern)^{mp} + \exp(-d/2\sigma^2)]$,
  the linear classifier $f_{\what}$ has $\ell_{\infty}^\eps$-robust classification error at most $\beta$ if
  \[
  \eps \leq \frac{1}{\sqrt{d}} \frac{\tchern \sqrt{dm} - \sqrt{d + 2\sigma^2 m \log 2}}
  {(1 + c)\sqrt{m} + 2\sigma} -
  \sigma \sqrt{\frac{2 \log \sfrac{1}{\beta}} {d} }
  \]
\end{lemma}

\begin{proof}
We follow the approach used for Theorem 21 in \cite{Schmidt18moredata}.
Define
\[
\alpha = 
\frac{\tchern \sqrt{dm} - \sqrt{d + 2\sigma^2 m \log 2}}
{(1 + c)\sqrt{m} + 2\sigma} 
\]
so that we can rewrite
\[
\eps \leq \frac{1}{\sqrt{d}} \alpha -
\sigma \sqrt{\frac{2 \log \sfrac{1}{\beta}} {d}}
\]
By Lemma \ref{lem:normalized_inner_product}, 
we have that $\ip{\what, \thetastar} \geq \alpha$
with probability at least $ 1 - [\exp(-6\sqrt{d}/5) + (\tailprobchern)^{mp} + \exp(-d/2\sigma^2)]$.

  \[
    \exp\parens*{-\frac{\parens{\ip{\what, \thetastar} - \eps \sqrt{d}}^2}{2 \sigma^2}} \; .
  \]
Since
\[
\ip{\what, \thetastar} - \eps \sqrt{d} \geq \alpha - \parens*{\frac{1}{\sqrt{d}} \alpha -
\sigma \sqrt{\frac{2 \log \sfrac{1}{\beta}} {d}}
} \sqrt{d}
= \sigma \sqrt{2 \log \sfrac{1}{\beta} } ,
\]
the robust classification error is bounded above by $\beta$, as desired.
\end{proof}

\begin{lemma}
	\label{lem:uatft_sample_complexity_1}
  Assume $\sigma \leq \frac{1}{32}d^{1/4}$ and $p > 0.99$.
Then, with probability at least $1 - [\exp(-6\sqrt{d}/5) + (\tailprobchern)^{mp} + \exp(-d/2\sigma^2)]$
the linear classifier $f_{\what}$ has $\ell_p^\eps$-robust classification error at most $0.01$ if
  \[
  m \; \geq \; \begin{cases} 100 \quad &\text{ for } \;\; \eps \, \leq \, \splitpoint d^{-\sfrac{1}{4}} \\
  256 \, \eps^2\sqrt{d} & \text{ for } \; \; \splitpoint d^{-\sfrac{1}{4}} \, \leq \, \eps \, \leq \, \epsupper \end{cases} \; .
  \]
\end{lemma}
\begin{proof}
We first apply Lemma \ref{lem:uatft_error} which gives a $\ell_\infty^{\eps'}$-robust classification error at most $\beta = 0.01$ for
\begin{align*}
\eps' &= \frac{1}{\sqrt{d}} \frac{\tchern \sqrt{dm} - \sqrt{d + 2\sigma^2 m \log 2}}
  {(1 + c)\sqrt{m} + 2\sigma} -
  \sigma \sqrt{\frac{2 \log \sfrac{1}{\beta}} {d} } \\
  &\geq \frac{1}{\sqrt{d}} \frac{\tchern \sqrt{dm} - \sqrt{d + 2\sigma^2 m \log 2}}
  {(1 + c)\sqrt{m} + 2\sigma} -
  \frac{1}{8}d^{-1/4}
\end{align*}

The remainder is simple algebraic manipulation.
First, we consider the case where $\eps \leq \splitpoint d^{-\sfrac{1}{4}}$. Using $m = 100$, we first bound 
\begin{align*}
c &= \frac{\sqrt{20} \sigma}{\sqrt{d}} \sqrt{\smfrac{1}{m} d^{\sfrac{1}{2}} + \log 2} \\
&\leq \frac{\sqrt{20}}{32} d^{-1/4} \sqrt{\smfrac{1}{100}d^{\sfrac{1}{2}} + \log 2 d^{\sfrac{1}{2}}} \\
&\leq \frac{\sqrt{20}}{32} d^{-1/4} \sqrt{d^{\sfrac{1}{2}}} \\
&\leq \frac{1}{5}
\end{align*}

The resulting robustness is
\begin{align*}
\eps' &\geq \frac{1}{\sqrt{d}} \frac{\tchern \sqrt{dm} - \sqrt{d + 2\sigma^2 m \log 2}}
  {(1 + c)\sqrt{m} + 2\sigma} -
  \frac{1}{8}d^{-1/4} \\
&= \frac{\tchern \sqrt{m} - \sqrt{1 + 2\sigma^2 md^{-1/2} \log 2}}
  {(1 + c)\sqrt{m} + 2\sigma} -
  \frac{1}{8}d^{-1/4} \\
&\geq \frac{\frac{7}{10}\sqrt{100} - \sqrt{1 + 200\sigma^2 d^{-1/2} \log 2}}
  {\frac{6}{5} \sqrt{100} + 2\sigma} -
  \frac{1}{8}d^{-1/4} \\
&\geq \frac{7 - \sqrt{1 + \frac{200}{32^2} \log 2}}
  {12 + \frac{1}{16}d^{1/4} } -
  \frac{1}{8}d^{-1/4} \\
&\geq \frac{7 - \sqrt{1 + \frac{200}{32^2} \log 2}}
  {\left( 12 + \frac{1}{16} \right) d^{1/4} } -
  \frac{1}{8}d^{-1/4} \\
&\geq \frac{1}{4}d^{-1/4} \\
&\geq \eps
\end{align*}

Next, we consider the case where $\splitpoint d^{-1/4} \leq \eps \leq \epsupper$. We again bound $c$:

\begin{align*}
c &= \frac{\sqrt{20} \sigma}{\sqrt{d}} \sqrt{\frac{\sqrt{d}}{m} + \log 2} \\
&\leq \frac{\sqrt{20}}{32} d^{-1/4} \sqrt{ \frac{1}{\mconst^2\eps^2} + \log 2 } \\
&\leq \frac{\sqrt{20}}{32} d^{-1/4} \sqrt{ \frac{4^2\sqrt{d}}{\mconst^2} + \log 2 } \\ 
&\leq \frac{\sqrt{20}}{32} \sqrt{ \frac{4^2}{\mconst^2} + \frac{\log 2}{\sqrt{d}} } \\
&\leq \frac{1}{5}
\end{align*}

The resulting robustness is
\begin{align*}
\eps' &\geq \frac{1}{\sqrt{d}} \frac{\tchern \sqrt{dm} - \sqrt{d + 2\sigma^2 m \log 2}}
  {(1 + c)\sqrt{m} + 2\sigma} -
  \frac{1}{8}d^{-1/4} \\
&\geq \frac{1}{\sqrt{d}} \frac{\tchern \sqrt{\mconst^2\eps^2d^{\sfrac{3}{2}}} - \sqrt{d + 2 \frac{1}{32^2} (\mconst^2\eps^2)d \log 2}}
  {\mconst(1 + c)\eps d^{\sfrac{1}{4}} + \frac{1}{16} d^{\sfrac{1}{4} } } -
  \frac{1}{2}\eps \\
&= \frac{\tchern (\mconst\eps) - \sqrt{ d^{-\sfrac{1}{2}} + \frac{\mconst^2}{32^2} \frac{\log 2}{2} d^{-\sfrac{1}{2}} \eps^2}}
  {\mconst(1 + c)\eps + \frac{1}{16} } -
  \frac{1}{2}\eps \\
&\geq \frac{\tchern (\mconst\eps) - \sqrt{d^{-\sfrac{1}{2}} + \frac{\mconst^2}{32^2 4^2} \frac{\log 2}{2} d^{-\sfrac{1}{2}} }}
  {4(1 + c) + \frac{1}{16} } -
  \frac{1}{2}\eps \\
&\geq \frac{11.5\eps -  4 \eps \sqrt{ 1 + \frac{\log 2}{128} }}
  {4 \frac{6}{5} + \frac{1}{16} } -
  \frac{1}{2}\eps \\
&\geq \eps
\end{align*}
as desired.
\end{proof}

\begin{corollary}
  Let $(x_0, y_0)$ and $(x_1, y_1), \ldots, (x_m, y_m)$ be drawn i.i.d.\ from a $(\thetastar, \sigma)$ Gaussian model with corruption parameter $p$ and mean norm $\sqrt{d}$.
  Let $\what_{sup} = y_0 x_0$.
  Let $\vzbar \in \RR^d$ be the sample mean $\vzbar = \frac{1}{m} \sum_{i=1}^{m} \yhi x_i$, where $\yhi = f_{\what_{sup}}(x_i)$.
  Let the UAT-FT estimator $\what \in \RR^d$ be the unit vector in the direction of $\vzbar$, i.e., 
  $\what = \sfrac{\vzbar}
  {\norm{\vzbar}_2}$.
  Assume $\sigma \leq \frac{1}{32}d^{1/4}$.
Then, with probability at least $1 - [\exp(-\frac{6\sqrt{d}}{5}) + (0.996)^{m} + \exp(-\frac{d}{2\sigma^2}) - 2 \exp (-\frac{d}{8\sigma^2 + 1)} )]$,
the linear classifier $f_{\what}$ has $\ell_p^\eps$-robust classification error at most $0.01$ if 
  \[
  m \; \geq \; \begin{cases} 100 \quad &\text{ for } \;\; \eps \, \leq \, \splitpoint d^{-\sfrac{1}{4}} \\
  256 \, \eps^2\sqrt{d} & \text{ for } \; \; \splitpoint d^{-\sfrac{1}{4}} \, \leq \, \eps \, \leq \, \epsupper \end{cases} \; .
  \]
\end{corollary}
Using the given restriction on $\sigma$, we can invoke Corollary 19 from \cite{Schmidt18moredata} with $\beta = 0.01$.
Thus, with probability at least $1 - 2 \exp (-\frac{d}{8\sigma^2 + 1)} )$, the classification error $p$ of the base classifier $f_{\what_{sup}}$ is less than $0.01$.
Conditioning on this event, we can invoke Lemma \ref{lem:uatft_sample_complexity_1} with $m$ as given, which yields a robust classification error of $f_{\what}$ of at most 0.01, 
with probability at least $1 - [\exp(-\frac{6\sqrt{d}}{5}) + (0.996)^{m} + \exp(-\frac{d}{2\sigma^2})$.

A union bound gives the desired total failure probability.

\end{document}